\documentclass{article} 
\usepackage{iclr2027_conference,times}

\usepackage{amsmath,amsfonts,bm}

\def\eqref#1{Eq.~(\ref{#1})}

\def\1{\bm{1}}

\DeclareMathAlphabet{\mathsfit}{\encodingdefault}{\sfdefault}{m}{sl}
\SetMathAlphabet{\mathsfit}{bold}{\encodingdefault}{\sfdefault}{bx}{n}

\usepackage{hyperref}
\usepackage{url}
\usepackage[table]{xcolor}   
\usepackage[skins]{tcolorbox}
\usepackage{textcomp}

\definecolor{promptbg}{RGB}{248,248,248}
\definecolor{promptrule}{RGB}{100,100,100}
\definecolor{promptgold}{RGB}{176,128,16}
\definecolor{promptthink}{RGB}{63,139,65}
\definecolor{promptaction}{RGB}{177,44,115}
\newtcolorbox{prompttpl}[2][]{
  enhanced, colback=promptbg, colframe=promptrule,
  colbacktitle=promptrule, coltitle=white,
  boxrule=0.5pt, arc=5pt,
  left=8pt, right=8pt, top=6pt, bottom=6pt,
  toptitle=3pt, bottomtitle=3pt,
  fonttitle=\bfseries\normalsize, fontupper=\normalsize\rmfamily,
  before skip=0pt, after skip=0pt, title={#2}, #1}
\newcommand{\promptvar}[1]{\textcolor{promptgold}{\{#1\}}}
\newcommand{\promptthinktags}{\textcolor{promptthink}{\textless think\textgreater\ \textless/think\textgreater}}
\newcommand{\promptactiontags}{\textcolor{promptaction}{\textless action\textgreater\ \textless/action\textgreater}}
\usepackage{graphicx}        
\usepackage{subcaption}      
\usepackage{booktabs}        
\usepackage{multirow}        
\usepackage{algorithm}
\usepackage{algpseudocode}

\newcommand{\std}[1]{\ensuremath{_{\pm #1}}}
\usepackage{amsmath}    
\usepackage{amssymb}    
\usepackage{amsthm}     
\newtheorem{proposition}{Proposition}[section]
\newtheorem{lemma}[proposition]{Lemma} 
\usepackage{enumitem}
\usepackage{wrapfig}
\newcommand{\xu}[1]{\textcolor{black}{#1}}

\title{GraphHCA: Closed-Form Hindsight Credit Assignment for Long-Horizon LLM Agents}

\author{%
  \textbf{Haodong Zhu}$^{1,2}$\thanks{Equal contribution.} \quad
  \textbf{Yangyang Ren}$^{1,2}$\footnotemark[1] \quad
  \textbf{Changbai Li}$^{1}$ \quad
  \textbf{Sheng Xu}$^{3}$\thanks{Corresponding authors:
    Sheng Xu (\texttt{shengxu@cuc.edu.cn}).} \quad
  \textbf{Linlin Yang}$^{3}$ \\
  \textbf{Haiguang Liu}$^{2}$ \quad
  \textbf{Baochang Zhang}$^{1,4}$ \\[0.5em]
  $^{1}$Beihang University \quad
  $^{2}$Zhongguancun Academy \quad
  $^{3}$Communication University of China \\
  $^{4}$Hangzhou Innovation Institute of Beihang University
}

\iclrfinalcopy 

\begin{document}

\maketitle

\ificlrfinal
  \lhead{}
\fi

\begin{abstract}
Group-based reinforcement learning (RL) has advanced large language models (LLMs) and is increasingly extending to agentic tasks, where sparse terminal rewards make step-level credit assignment essential. Existing methods assign credit from what follows an action in sampled rollouts, but do not explicitly capture its retrospective relation to the realized outcome. Hindsight credit assignment (HCA) instead attributes credit through the ratio of hindsight to behavior-policy probabilities, but estimating the hindsight distribution requires an auxiliary model or an extra pass.
To address this estimation bottleneck, we propose \textbf{GraphHCA}, a model-free realization of HCA that eliminates explicit hindsight-distribution estimation. 
For terminal-goal tasks with deterministic transitions, Bayes' rule reduces the hindsight ratio to a ratio of behavior-policy success probabilities at consecutive states. 
Taking logs yields a state-wise \emph{success potential}, whose increment across a transition provides step-level credit.
GraphHCA estimates this potential from pooled rollouts through a discounted recursion on the induced transition graph, which admits a unique fixed point on any directed graph. 
The resulting step-level signal is combined with the trajectory-level advantage, requiring neither a learned hindsight model nor an extra forward pass and recovering GRPO when the step-level weight is zero.
Among all compared baselines, {GraphHCA} achieves state-of-the-art results on ALFWorld and WebShop at both LLM scales, and on Sokoban with a vision-language agent. 
For example, on ALFWorld it improves overall success rate by up to 24.6 points over GRPO and by up to 4.7 points over the strongest step-level baseline.
\end{abstract}
\section{Introduction}

Large language models (LLMs) are increasingly deployed as controllers of interactive agents~\citep{openai2024gpt4technicalreport,geminiteam2025geminifamilyhighlycapable,guo2025deepseek}, performing embodied household tasks~\citep{shridhar2020alfworld}, navigating web and mobile environments~\citep{yao2022webshop,rawles2023androidinthewild}, and invoking external tools for complex reasoning~\citep{yao2023react,shinn2024reflexion,schick2023toolformer}.
These tasks typically require long-horizon interaction and sequential decision making over large natural-language action spaces.

A fundamental challenge in applying reinforcement learning (RL)~\citep{reinforcement-learning,sutton1999policy} to such agents is the sparsity of outcome-based rewards.
Many agentic tasks provide only a scalar reward at the end of an episode, leaving intermediate actions without direct supervision.
This gives rise to the \textit{credit assignment} problem: determining how a terminal outcome should be attributed to the individual decisions that produced it.
The problem becomes increasingly difficult as interaction horizons and action spaces grow.

Conventional actor--critic methods address credit assignment through learned value estimates~\citep{schulman2017proximal}, but training a critic at LLM scale is computationally expensive and particularly challenging under sparse rewards.
Group-based methods instead sample a group of rollouts for each task, as illustrated in Fig.~\ref{fig:head}(a), and estimate advantages directly from their outcomes without an explicit critic.
%
%
GRPO~\citep{grpo}, for example, assigns the same trajectory-level advantage to all actions within a rollout, so intermediate decisions are not distinguished by their individual contributions to the final outcome.
This coarse attribution has motivated step-level variants: GiGPO~\citep{gigpo} compares transitions from shared states using the outcomes of their respective rollouts, whereas GraphGPO~\citep{graphgpo} scores transitions by shortest-path progress in the pooled rollout graph.
Although these methods provide finer-grained signals, their credit is still derived primarily from forward evidence contained in sampled continuations, without explicitly modeling the retrospective relation between an intermediate action and the realized outcome.
As illustrated in Fig.~\ref{fig:head}(c), a failed rollout may still contain useful intermediate decisions, while a successful rollout may include redundant or erroneous ones; likewise, shortest-path progress need not identify the continuation with the highest empirical expected success.
\xu{HCA~\citep{hca} attributes credit retrospectively through the hindsight-to-policy ratio, but estimating the hindsight distribution is expensive: it requires either an auxiliary model or, as in HCAPO~\citep{hcapo}, an extra outcome-conditioned policy pass.}
%
%
Fig.~\ref{fig:head}(b) shows that the additional model pass of HCAPO raises its credit-assignment time to about $10$ seconds, the highest among all methods.
%
\xu{We therefore face a dilemma: step-level credit is either cheap but only forward-looking (Fig.~\ref{fig:head}(c)), or retrospective but too expensive to obtain (Fig.~\ref{fig:head}(b)).}

\begin{figure}[!t]
    \centering
    \includegraphics[width=1.0\linewidth]{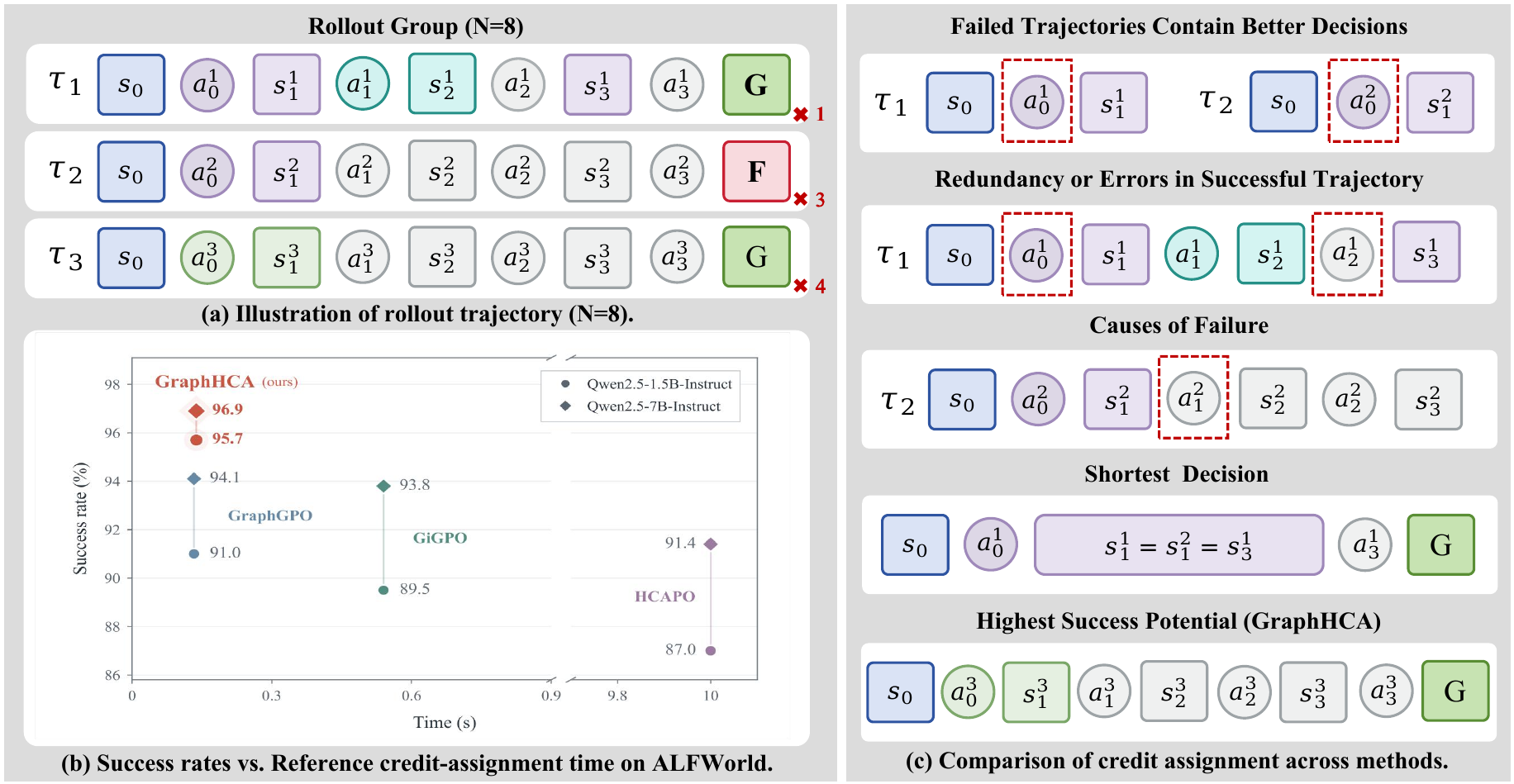}
    \caption{\textbf{Comparison of credit-assignment principles.} 
    (a) Illustration of rollout trajectory (N=8). Squares and circles denote states and actions; matching non-gray state colors indicate identical states. 
    (b)  ALFWorld success rates versus reference credit-assignment time. 
    (c) Comparison of credit assignment across methods, highlighting the limitations of relying solely on trajectory outcomes or shortest paths.}
    \label{fig:head}
    \vspace{-15pt}
\end{figure}


\textbf{Our key observation is that, in terminal-goal tasks with deterministic transitions, the hindsight ratio admits a closed-form expression, \xu{so it need not be estimated by an auxiliary model or an extra forward pass.}}
By Bayes' rule, it reduces to the ratio of behavior-policy success probabilities at the successor and current states, eliminating the need to model the hindsight distribution explicitly \xu{and thus avoiding the outcome-conditioned pass that makes HCAPO's credit-assignment time the highest in Fig.~\ref{fig:head}(b).}
We define the logarithm of this state-wise success probability as the \emph{success potential}, whose change across a transition corresponds to the log hindsight ratio and serves as a dense step-level credit signal.
Hindsight credit assignment is thus reduced from estimating a distribution over the combinatorial action space to estimating a single scalar per state.

\xu{Having reduced hindsight credit to estimating a single state-wise scalar, the remaining challenge is to estimate that scalar from a finite group of sampled rollouts.}
We therefore propose \textbf{GraphHCA}, a model-free realization of HCA that estimates the success potential directly from sampled rollouts, \xu{as illustrated in Fig.~\ref{fig:graphhca-overview}.}
Because each rollout observes only one sampled continuation at a visited state, GraphHCA pools a group of rollouts into a shared transition graph, merges identical states, and propagates terminal outcomes according to empirical action frequencies.
To recover the path-length information lost through state merging, we introduce a propagation discount, which also makes the resulting expectation backup contractive and guarantees a unique fixed point even on cyclic graphs.
The resulting potential differences provide transition-level credit, which is standardized among actions leaving the same state and combined with the trajectory-level advantage for policy optimization.
GraphHCA thus extracts dense hindsight credit from rollouts already collected for training, without learning an auxiliary hindsight model or requiring an additional outcome-conditioned forward pass, and reduces to GRPO when the step-level term is disabled.
Our contributions are summarized as follows:

\begin{itemize}

\item \textbf{Closed-form hindsight credit.}
We show that, for terminal-goal tasks with deterministic transitions, HCA reduces to state-wise success probabilities, replacing hindsight-distribution estimation with a single scalar per state.

\item \textbf{Model-free graph estimation.}
We estimate this scalar directly from pooled rollouts through a discounted expectation recursion on the transition graph, with a unique fixed point on arbitrary directed graphs, \xu{and requires no learned critic, auxiliary hindsight model, or extra forward pass.}

\item \textbf{Consistent empirical gains.}
\xu{Across ALFWorld~\citep{shridhar2020alfworld}, WebShop~\citep{yao2022webshop}, and $6\times6$ Sokoban~\citep{SchraderSokoban2018}, GraphHCA achieves the best performance among all compared baselines across both LLM scales and the vision-language setting. For example, GraphHCA improves success by up to 24.6 points over GRPO and by {up to 4.7} points over the strongest step-level baseline on ALFWorld.}
%

\end{itemize}

\begin{figure}[!t]
    \centering
    \includegraphics[width=1.0\linewidth]{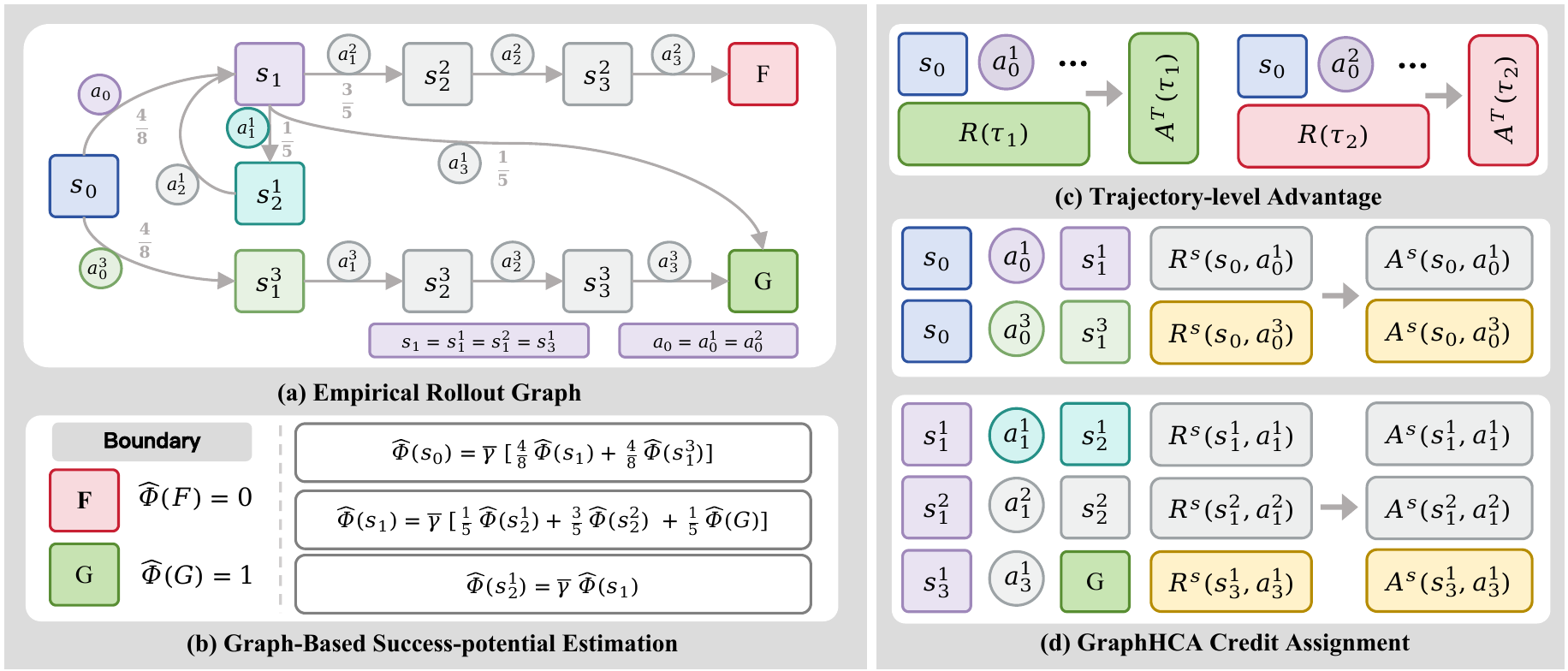}
    \caption{\textbf{Overview of GraphHCA.}
    \textbf{(a)} Rollouts are merged into a graph where identical states share a node and edges carry action frequencies.
    \textbf{(b)} The success  probability $\hat\Phi$ is the fixed point of a discounted recursion with $\hat\Phi(F)=0$ and $\hat\Phi(G)=1$.
    \textbf{(c)} The trajectory-level advantage $A^{T}(\tau)$ is shared by all steps of $\tau$.
    \textbf{(d)} GraphHCA takes the increment of $\log\hat\Phi$ as the step reward $R^{s}$ and standardizes it among transitions sharing a source state to obtain $A^{s}$.}
    \label{fig:graphhca-overview}
    \vspace{-10pt}
\end{figure}

\section{Related Work}
\textbf{RL for LLM agents.}
Reinforcement learning has become a central paradigm for LLM post-training, progressing from RL from human feedback~\citep{ziegler2019fine,stiennon2020learning,ouyang2022training,bai2022training,rafailov2023direct} to RL with verifiable rewards~\citep{guo2025deepseek,team2025kimi,grpo}, with automatically checkable feedback supporting reasoning, code, search, and tool use~\citep{le2022coderl,jin2025search,qian2025toolrl}.
For long-horizon LLM agents in interactive environments~\citep{yao2023react,shinn2024reflexion,schick2023toolformer,shridhar2020alfworld,yao2022webshop,rawles2023androidinthewild,yang2024swe,trivedi2024appworld}, early approaches often relied on learned critics, including PPO-style and agent-specific hierarchical or search-augmented methods~\citep{schulman2017proximal,peng2019advantage,zhou2024archer,putta2024agent,bai2024digirl}.
At LLM scale, however, critic training is costly and unstable under sparse terminal rewards, motivating critic-free advantage estimation from sampled rollouts.

\textbf{Step-level credit assignment.}
Group-based methods replace learned critics with rollout statistics, but GRPO~\citep{grpo} broadcasts a single trajectory-level advantage across steps.
Efforts to refine credit follow two main directions.
Dense shaping adds stepwise signals: process reward models~\citep{lightman2023let} require extra supervision, intrinsic rewards~\citep{empg} rely on surrogate objectives, and potential-based shaping~\citep{ng1999policy} specifies a policy-invariant reward form but leaves the potential unspecified.
Rollout-based methods use forward evidence from sampled continuations: GiGPO~\citep{gigpo} compares transitions from shared states, while GraphGPO~\citep{graphgpo} scores shortest-path progress on pooled rollout graphs.
Hindsight credit assignment (HCA)~\citep{hca} instead formalizes outcome-conditioned credit through the hindsight-to-policy ratio, but estimating the hindsight distribution typically requires an auxiliary model or an extra outcome-conditioned policy pass~\citep{hcapo}.
GraphHCA removes this bottleneck by deriving the ratio in closed form for deterministic terminal-goal tasks and estimating the resulting state-wise success potential directly from pooled rollouts.

\section{Preliminaries}
\label{sec:prelim}

\paragraph{RL for long-horizon LLM agents.}
We model an interactive agent task as a finite-horizon decision process.
Let $x$ denote the task instruction and $s_t$ the environment state at step $t$.
At each step, the LLM policy $\pi_\theta$ samples a natural-language action
$a_t\sim\pi_\theta(\cdot\mid s_t,x)$, after which the environment transitions to $s_{t+1}$.
A rollout yields $\tau=(s_0,a_0,\dots,a_{T-1},s_T)$ under the behavior policy
$q(a\mid s,x)=\pi_{\theta_{\mathrm{old}}}(a\mid s,x)$ and receives the terminal reward
$R(\tau)=\mathbf 1[s_T\in\mathcal G]$, where $\mathcal G$ and $\mathcal F$ denote the goal and failure sets.
Since all rollouts within a group share the same task $x$, we suppress this common conditioning below and write
$q(a\mid s)$ and $\pi_\theta(a\mid s)$ for brevity.
Maximizing $J(\theta)=\mathbb E_{\tau\sim\pi_\theta}[R(\tau)]$ gives the policy gradient
$
\nabla_\theta J(\theta)=\mathbb E_{\tau\sim\pi_\theta}\!\left[\sum\nolimits_{t=0}^{T-1}A_t\,\nabla_\theta\log\pi_\theta(a_t\mid s_t)\right],
\label{eq:pre-pg}
$
which requires a per-step advantage $A_t$, whereas the terminal reward reveals only whether the task succeeded, not which intermediate decisions should receive credit for the outcome.

\paragraph{Group-based advantage estimation.}
Classically, $A_t$ is estimated by a learned critic~\citep{schulman2017proximal}, which is costly and can be unstable at LLM scale.
Group-based methods instead derive advantages directly from rollout statistics.
GRPO~\citep{grpo} samples $N$ trajectories per task under $q$ and assigns the standardized terminal return
\begin{equation}
A^{\mathrm{T}}(\tau_i)=\frac{R(\tau_i)-\mu_R}{\sigma_R}
\label{eq:pre-grpo}
\end{equation}
to every step of $\tau_i$, leaving intermediate credit dependent on subsequent decisions.
For finer credit, GiGPO~\citep{gigpo} compares transitions from shared states using their trajectory outcomes, while GraphGPO~\citep{graphgpo} scores transitions by shortest-path progress on the pooled rollout graph.
Both still derive credit from forward evidence in sampled futures, without explicitly conditioning the action distribution on the realized outcome.

\paragraph{Hindsight credit assignment.}
HCA~\citep{hca} takes a retrospective view, defining step-level credit through the relation between an action and a realized future outcome.
Its \emph{hindsight distribution} $h(a_t\mid s_t,s_k)$ gives the probability of taking $a_t$ at $s_t$ conditioned on later reaching $s_k$.
The ratio to the behavior policy,
\begin{equation}
\rho(s_t,a_t,s_k)=\frac{h(a_t\mid s_t,s_k)}{q(a_t\mid s_t)},
\label{eq:pre-hca}
\end{equation}
upweights actions whose probability increases under outcome conditioning and downweights those whose probability decreases.
Its practical difficulty lies in estimating $h$, which classical HCA learns with an auxiliary model over the action space, while HCAPO~\citep{hcapo} approximates it through an additional outcome-conditioned pass of the policy LLM.

\section{Method}
\label{sec:method}

We propose \textbf{GraphHCA}, as illustrated in Fig.~\ref{fig:graphhca-overview}, a model-free realization of hindsight credit assignment (HCA) for long-horizon LLM agents.
Our key observation is that, for terminal-goal task with deterministic transitions, the hindsight ratio admits a closed form in terms of state-wise success probabilities.
We define the logarithm of this success probability as the \emph{success potential}, estimate it model-free on a pooled rollout graph, and use its transition differences as step-level credit for policy optimization.
The following subsections derive the potential (Sec. \ref{sec:potential}), estimate it on the rollout graph (Sec. \ref{sec:graph}), integrate it into policy optimization (Sec. \ref{sec:po}), and analyze its properties (Sec. \ref{sec:theory}).
The detailed algorithm is presented in Appendix~\ref{app:algrithom}.

\subsection{The Success Potential from Bayesian Hindsight}
\label{sec:potential}

\paragraph{Hindsight as success-posterior inference.}
We specialize HCA to terminal-goal tasks with binary outcomes.
Let $\mathcal G$ and $\mathcal F$ denote the terminal goal and failure sets, and let $G=\{s_T\in\mathcal G\}$ denote the success event.
Rather than conditioning on an arbitrary future state $s_k$, we condition on $G$, so that $h(a_t\mid s_t,s_k)$ becomes the success-conditioned posterior $h(a\mid s,G)$.
By Bayes' rule,
\begin{equation}
h(a\mid s,G)
=\frac{q(a\mid s)\,P(G\mid s,a)}{P(G\mid s)},
\qquad
P(G\mid s)=\sum_a q(a\mid s)\,P(G\mid s,a).
\label{eq:bayes-posterior}
\end{equation}
Here, $P(G\mid s)$ is the marginal probability of eventual success from $s$.
For the exact analysis, we assume that the conditional probability of eventual success under $q$ depends only on the current environment state, independently of the preceding history and time step.
Under deterministic transitions, each action $a$ uniquely determines a successor $s'_a=f(s,a)$; together with the preceding assumption, this gives $P(G\mid s,a)=P(G\mid s'_a)$.
The action-dependent likelihood in the hindsight posterior is therefore characterized by the success probability of its successor state.

\paragraph{The success potential.}
Define the state-wise success probability
\begin{equation}
\Phi(s)=P(G\mid s),
\qquad
\Phi(s)=1\ \text{for }s\in\mathcal G,\quad
\Phi(s)=0\ \text{for }s\in\mathcal F,
\end{equation}
and its log-scale \emph{success potential}
\begin{equation}
\Psi(s)=\log\max\{\Phi(s),\epsilon\},
\qquad \epsilon\in(0,1).
\label{eq:psi-def}
\end{equation}
The floor only keeps the logarithmic representation finite when $\Phi(s)=0$ and does not alter the underlying success probability.
Under the terminal reward $R(\tau)=\mathbf 1[s_T\in\mathcal G]$, $\Phi$ is also the value function of the behavior policy $q$.

\begin{proposition}[Closed-form hindsight ratio]
\label{prop:closed-form}
Under the preceding assumption and deterministic transitions, the success probability at any non-terminal state satisfies
\begin{equation}
\Phi(s)
=\sum_a q(a\mid s)\Phi(s'_a)
=\mathbb E_{a\sim q}\!\left[\Phi(s'_a)\right].
\label{eq:phi-exact}
\end{equation}
For any state with $\Phi(s)>0$ and action with $q(a\mid s)>0$, the hindsight ratio specialized to the success event admits the closed form
\begin{equation}
\rho(s,a)
=\frac{h(a\mid s,G)}{q(a\mid s)}
=\frac{P(G\mid s,a)}{P(G\mid s)}
=\frac{\Phi(s'_a)}{\Phi(s)}.
\label{eq:rho-closed}
\end{equation}
\end{proposition}

The closed form in~\eqref{eq:rho-closed} avoids explicit estimation
of the hindsight distribution.
Using the floored potential $\Psi$ in~\eqref{eq:psi-def}, we define
the step credit $\ell(s,a)$ as
\begin{equation}
\ell(s,a) = \Psi(s'_a) - \Psi(s).
\label{eq:credit-exact}
\end{equation}
This credit remains finite when the underlying potential vanishes
and equals $\log\rho(s,a)$ whenever the floor is inactive at both
$s$ and $s'_a$.
Thus, hindsight credit assignment reduces to estimating the scalar
state potential $\Phi$.
When the floor is inactive, the resulting additive credit admits
the log-likelihood-ratio interpretation analyzed in
Sec.~\ref{sec:theory}.

\subsection{Model-Free Estimation on the Rollout Graph}
\label{sec:graph}

\paragraph{Pooled rollout graph.}
The marginalization~\eqref{eq:phi-exact} characterizes the exact success probability $\Phi$, but sampled rollouts reveal only a subset of the available transitions.
Following~\cite{graphgpo}, we pool the rollout group into a shared graph, merging identical environment states across rollouts and time steps.
Each realized transition $s\xrightarrow{a}s'_a$ is assigned the empirical weight
$\hat q(a\mid s)=n(s,a)/n(s)$, where $n(s)$ counts visits to $s$ and $n(s,a)$ those on which action $a$ is taken.
Pooling enables evidence sharing across trajectories, but merging repeated occurrences of the same state removes their temporal positions and hence their path-length information.
Consequently, an undiscounted backup may assign no cost to additional transitions or detours.
We therefore define the graph value through the discounted expectation recursion
\begin{equation}
\hat\Phi(s)=\bar\gamma\,\mathbb E_{a\sim\hat q}\big[\hat\Phi(s'_a)\big],
\qquad \bar\gamma\in(0,1),
\label{eq:phi-backup}
\end{equation}
with boundary values $\hat\Phi=1$ on $\mathcal G$ and $\hat\Phi=0$ on $\mathcal F$.
Unlike the exact recursion~\eqref{eq:phi-exact}, $\bar\gamma$ is not inherited from the task objective, but acts as a propagation discount that restores sensitivity to path length after state merging.
The resulting $\hat\Phi$ is a discounted empirical surrogate for $\Phi$, aggregating visits with potentially different interaction histories and remaining horizons.

\begin{proposition}[Contraction and convergence]
\label{prop:contraction}
For any directed graph, including graphs with cycles, the operator in~\eqref{eq:phi-backup} is a $\bar\gamma$-contraction in the sup norm.
Hence, iteration from any initialization converges geometrically to a unique fixed point $\hat\Phi$, corresponding to the $\bar\gamma$-discounted expected terminal success under the empirical policy $\hat q$.
\end{proposition}

Using the fixed point $\hat\Phi$, we define the graph success potential
$\hat\Psi(s)=\log\max\{\hat\Phi(s),\epsilon\}$ and the corresponding raw transition credit
\begin{equation}
\hat\ell(s,a)
=
\hat\Psi(s'_a)-\hat\Psi(s)
\in
[-\log(1/\epsilon),\,\log(1/\epsilon)].
\label{eq:credit-graph}
\end{equation}

\subsection{Policy Optimization}
\label{sec:po}

\paragraph{Step-level advantage.}
We take the graph credit of Sec.~\ref{sec:graph} as the reward of a realized transition, $R^{\mathrm s}(s_t,a_t)=\hat\ell(s_t,a_t)$.
Following the state-level grouping of~\citep{gigpo}, we standardize it within the group $\mathcal B_{s}$ of all transitions leaving the same state $s$ across the $N$ rollouts,
\begin{equation}
A^{\mathrm{s}}(s_t,a_t)=\frac{R^{\mathrm s}(s_t,a_t)-\mu_{\mathcal B_{s_t}}}{\sigma_{\mathcal B_{s_t}}},
\qquad \mu_{\mathcal B_s}=\operatorname{mean}_{\mathcal B_s}(R^{\mathrm s}),\quad
\sigma_{\mathcal B_s}=\operatorname{std}_{\mathcal B_s}(R^{\mathrm s}).
\label{eq:step-adv}
\end{equation}
We set $A^{\mathrm{s}}=0$ at singleton states ($|\mathcal B_s|=1$).
For both trajectory- and state-level normalization, we also assign zero advantages to groups with zero variance.
For groups with positive variance, standardization preserves the ordering induced by $\hat\ell$ and therefore ranks the sampled actions from a common state by their successor potential $\hat\Psi(s'_a)$.

\paragraph{Combined objective.}
We add the step term to the trajectory-level advantage $A^{\mathrm{T}}(\tau)$ of~\eqref{eq:pre-grpo},
\begin{equation}
\hat A_t=A^{\mathrm{T}}(\tau)+w_{\mathrm{s}}\,A^{\mathrm{s}}(s_t,a_t),
\qquad w_{\mathrm{s}}\ge0,
\label{eq:combined-adv}
\end{equation}
the two being complementary, since $A^{\mathrm{T}}$ anchors the update to task success and stays defined at singleton states while $A^{\mathrm{s}}$ densifies credit within a trajectory.
Using $\hat A_t$ for the policy update at step $t$, we optimize the clipped surrogate with a reference-KL penalty,
{\small
\begin{equation}
\mathcal J(\theta)=\mathbb E\!\left[\frac{1}{\sum_i T_i}\sum_{i=1}^{N}\sum_{t}
\min\!\Big(\varrho_{i,t}(\theta)\,\hat A_{i,t},\;
\operatorname{clip}\big(\varrho_{i,t}(\theta),1-\varepsilon,1+\varepsilon\big)
\hat A_{i,t}\Big)\right]
-\beta_{\mathrm{KL}}\,D_{\mathrm{KL}}\!\big(\pi_\theta\,\|\,\pi_{\mathrm{ref}}\big),
\label{eq:objective_1}
\end{equation}
}
with $\varrho_{i,t}(\theta)=\pi_\theta(a_{i,t}\mid s_{i,t})/\pi_{\theta_{\mathrm{old}}}(a_{i,t}\mid s_{i,t})$.
Setting $w_{\mathrm{s}}=0$ recovers GRPO exactly, so GraphHCA is a strict, model-free refinement of it, adding no model of its own since $\hat\Psi$ is read directly off the pooled graph, whose $\hat\Phi$ draws on the whole group including sub-paths of failed trajectories that never reach the goal on their own.

\subsection{Properties of the GraphHCA Credit}
\label{sec:theory}

We further establish that the exact credit is the log-likelihood ratio to the hindsight posterior when the floor is inactive (Proposition~\ref{prop:kl}), and that the graph credit separates from GraphGPO's distance surrogate by aggregation rather than by reachability (Proposition~\ref{prop:fidelity}).
Proofs are in Appendix~\ref{app:proofs}.

\begin{proposition}[Hindsight KL identity and bounded step credit]
\label{prop:kl}
Under the assumptions of Proposition~\ref{prop:closed-form}, fix a non-terminal state $s$ such that $\Phi(s)\ge\epsilon$
and $\Phi(s'_a)\ge\epsilon$ for every action with $q(a\mid s)>0$.
Then, on the support of $q(\cdot\mid s)$, Proposition~\ref{prop:closed-form} gives
$\ell(s,a)=\log[h(a\mid s,G)/q(a\mid s)]$.
Therefore,
\begin{equation}
\mathbb E_{a\sim q(\cdot\mid s)}
\big[\ell(s,a)\big]
=
-\mathrm{KL}\big(
q(\cdot\mid s)\,\big\|\,h(\cdot\mid s,G)
\big)
\le 0,
\label{eq:kl-identity}
\end{equation}
with equality if and only if
$q(\cdot\mid s)=h(\cdot\mid s,G)$.
Independently, for the empirical graph, $\hat\Phi(s)\in[0,\bar\gamma]$ at every interior
node. Whether or not the floor is active, it gives
$\hat\Psi(s)\in[\log\epsilon,0]$ and
$|\hat\ell(s,a)|\le\log(1/\epsilon)$.
\end{proposition}

\begin{proposition}[Expectation versus best-case aggregation]
\label{prop:fidelity}
Let $R^{\mathrm G}(s,a,s'_a)$ denote GraphGPO's step reward, and $R^{\mathrm G}(s,a,s'_a)\propto \bar\gamma^{\,d(s'_a)+c(s,a)}$.
$d(\cdot)$ the shortest-path distance to $\mathcal G$ on the empirical graph, and adopt the unit-cost setting of~\citep{graphgpo}.
Replacing the expectation in~\eqref{eq:phi-backup} by a maximum under the same boundary defines the fixed point $\Phi^{\max}$, and
\begin{equation}
\hat\Phi(s)\ \le\ \Phi^{\max}(s)\ =\ \bar\gamma^{\,d(s)}
\qquad(s\ \text{goal-reachable}),
\label{eq:fidelity}
\end{equation}
with equality if and only if $\hat q$ is supported on the argmax successors at $s$ and downstream.
The two credits are therefore one backup under two aggregations, $\hat\ell$ ranking the actions at $s$ by $\hat\Psi(s'_a)$ and $R^{\mathrm G}$ by $d(s'_a)$.
\end{proposition}

Under the floor-inactive conditions of Proposition~\ref{prop:kl}, the exact step credit has a log-likelihood-ratio interpretation relative to the hindsight posterior.
Proposition~\ref{prop:fidelity} further distinguishes GraphHCA from GraphGPO by their aggregation rules: GraphHCA evaluates successors through empirical expected continuation values, whereas GraphGPO relies on shortest-path distance and thus reflects a best-case continuation.
Consequently, GraphHCA remains sensitive to transition frequencies and can distinguish equidistant successors whose distance-based scores coincide.
HCAPO~\citep{hcapo} targets the same hindsight ratio through an additional outcome-conditioned policy evaluation, while GraphHCA obtains its credit directly from the pooled rollout graph without an auxiliary hindsight model.

\section{Experiments}
\label{sec:experiments}

\subsection{Experimental Setup}
\label{sec:exp-setup}

\paragraph{Benchmarks.}
We evaluate GraphHCA on three challenging multi-turn agentic benchmarks: ALFWorld~\citep{shridhar2020alfworld} and WebShop~\citep{yao2022webshop} for LLM agents, and the $6\times6$ Sokoban game~\citep{SchraderSokoban2018} for a vision-language model (VLM) agent.
Following prior work~\citep{gigpo,graphgpo}, we report the success rate on ALFWorld, the task score and success rate on WebShop, and the success rate on Sokoban, averaged over three random seeds for all RL methods.
Detailed descriptions of the three benchmarks are provided in Appendix~\ref{app:benchmarks}.

\begin{figure}[!t]
\centering
\begin{subfigure}[b]{0.325\textwidth}
    \centering
    \includegraphics[width=\linewidth]{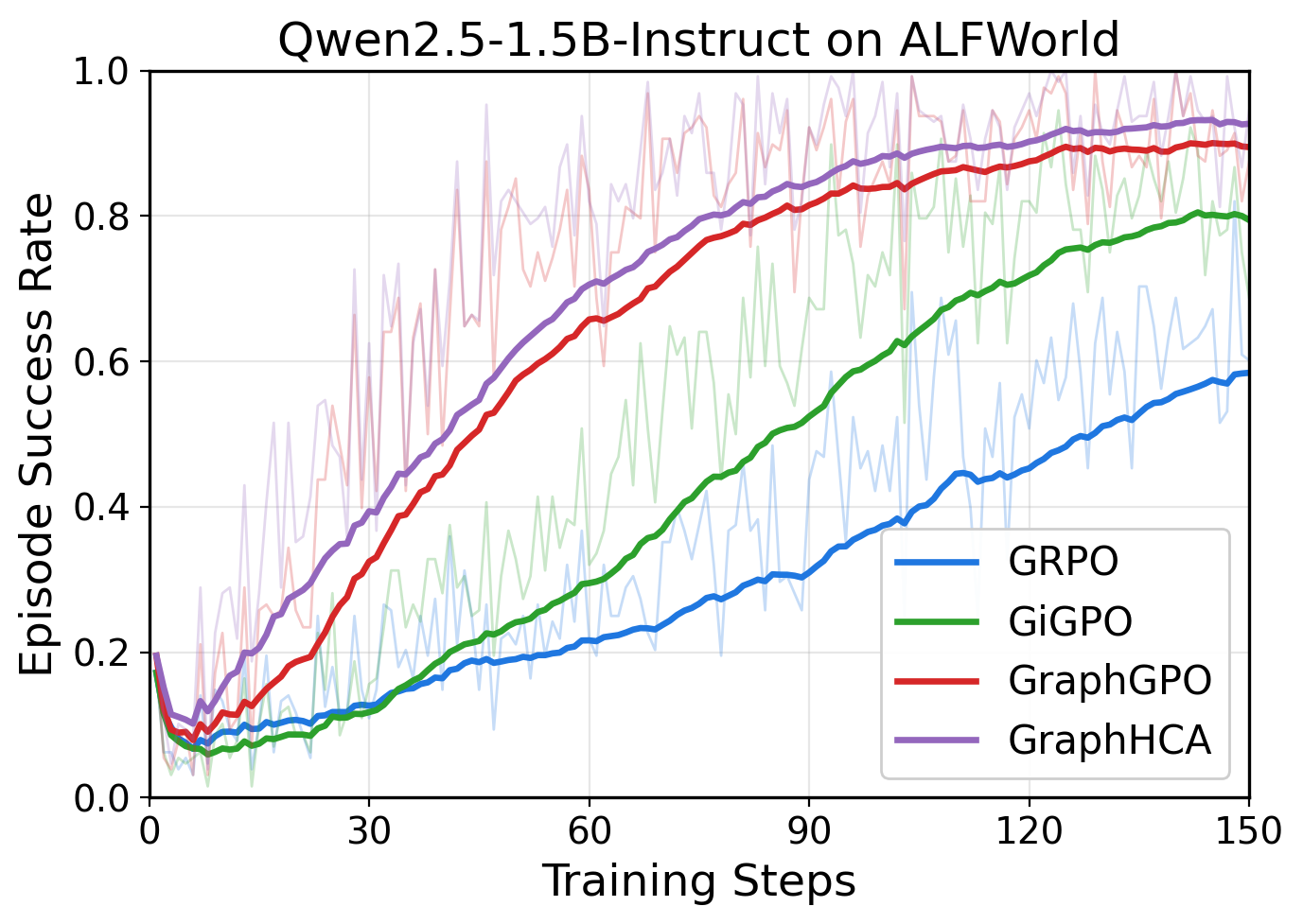}
    \caption{ALFWorld}
    \label{fig:curve-alfworld}
\end{subfigure}
\hfill
\begin{subfigure}[b]{0.325\textwidth}
    \centering
    \includegraphics[width=\linewidth]{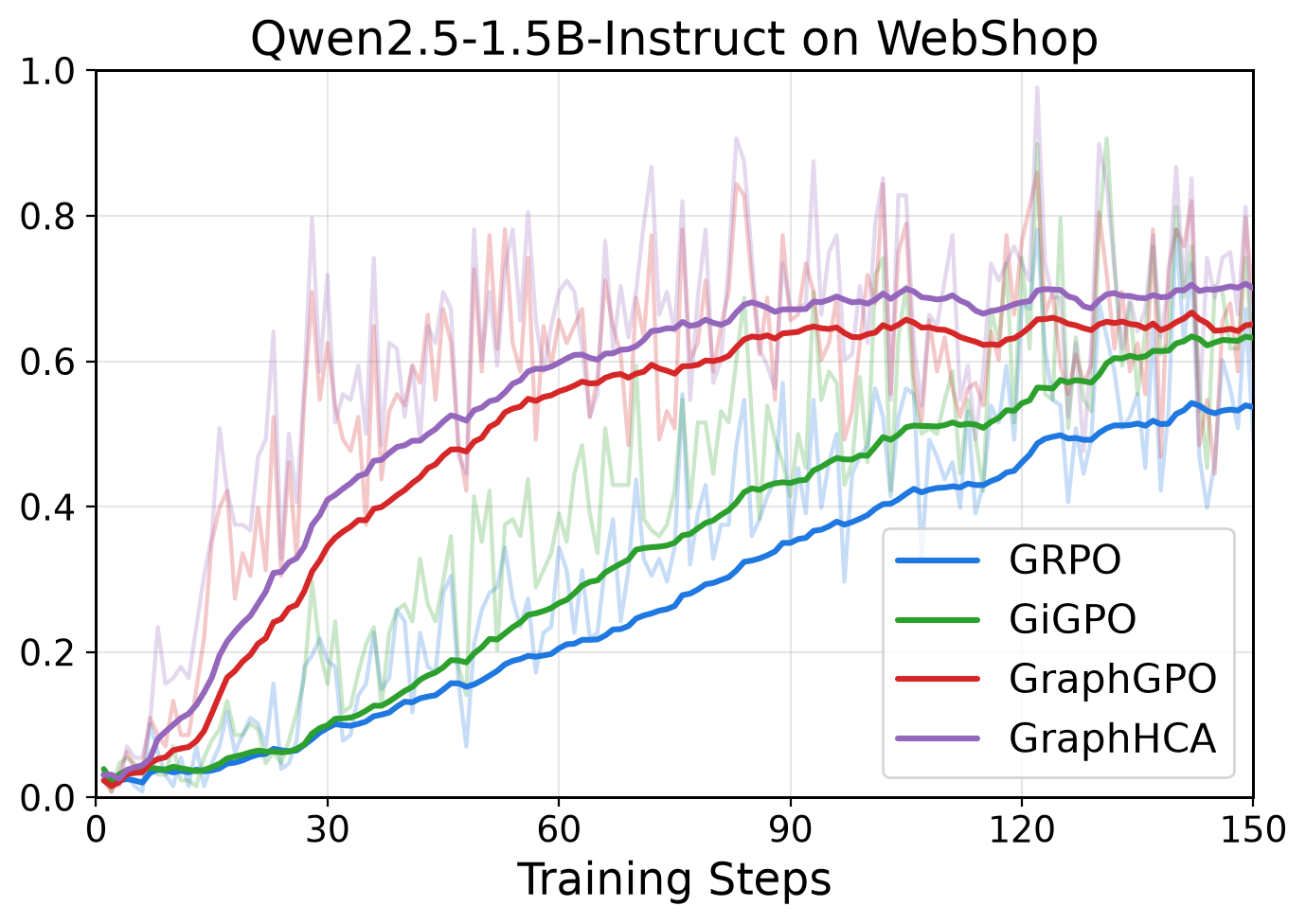}
    \caption{WebShop}
    \label{fig:curve-webshop}
\end{subfigure}
\hfill
\begin{subfigure}[b]{0.325\textwidth}
    \centering
    \includegraphics[width=\linewidth]{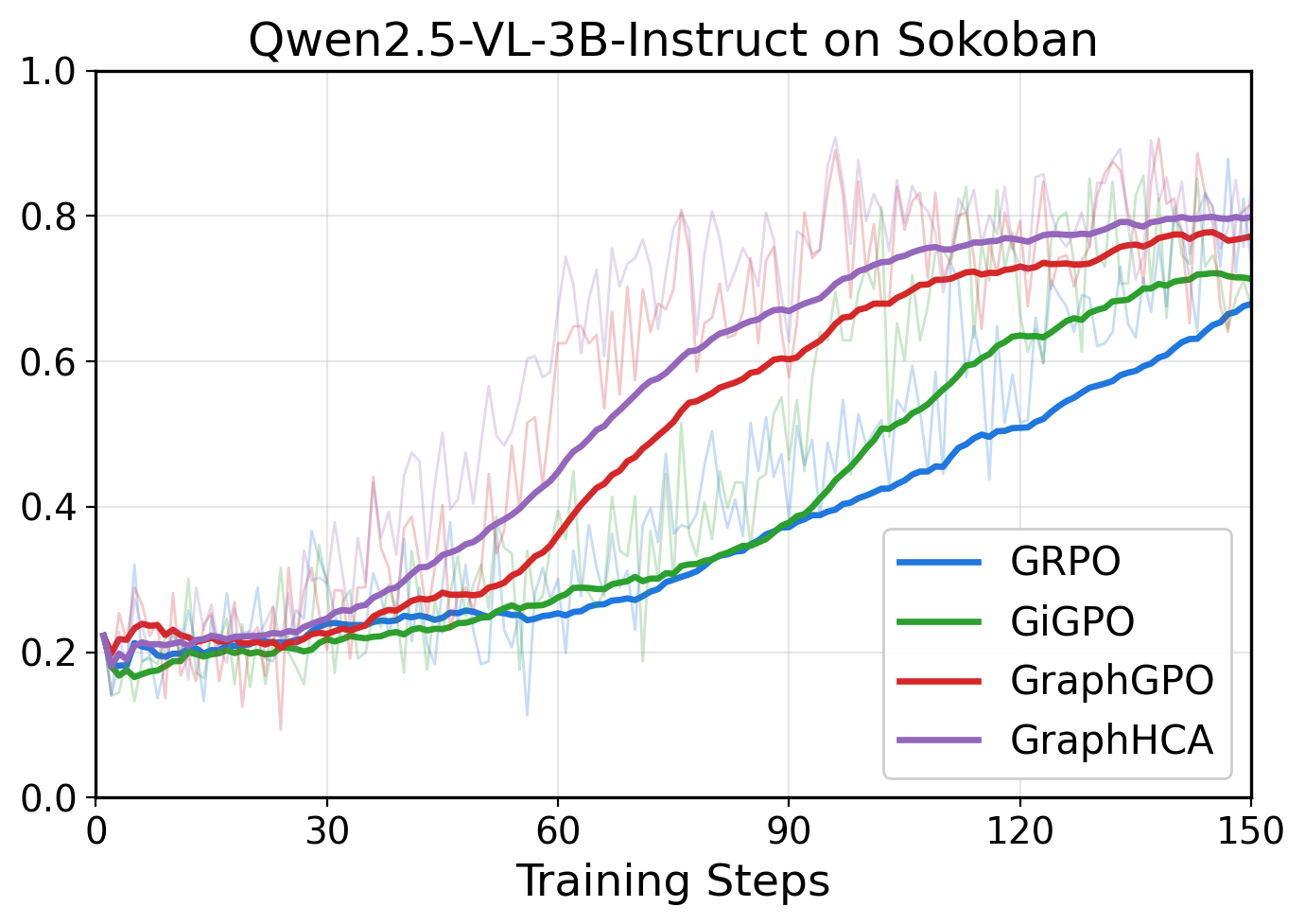}
    \caption{Sokoban}
    \label{fig:curve-sokoban}
\end{subfigure}
\caption{Training success rate over update steps on ALFWorld and WebShop with Qwen2.5-1.5B-Instruct, and on Sokoban with Qwen2.5-VL-3B-Instruct. GraphHCA converges faster and reaches higher final success than the fine-grained credit-assignment baselines across all three environments.}
\vspace{-10pt}
\label{fig:success_curves}
\end{figure}

\begin{table*}[!ht]
\centering
\small
\caption{Test performance on ALFWorld and WebShop. For ALFWorld, we report the average success rate (\%) for each subtask and the overall result. For WebShop, we report the average task score and the average success rate (\%). RL results are averaged over three random seeds. The best performance in each column is highlighted in \textbf{bold}.}
\label{tab:alfworld_webshop}
\resizebox{1.0\textwidth}{!}{
\setlength{\tabcolsep}{0.8mm}{
\renewcommand{\arraystretch}{1.1}
\begin{tabular}{ll|ccccccc|cc}
\toprule
\multirow{2}{*}{\textbf{Type}} & \multirow{2}{*}{\textbf{Method} }
& \multicolumn{7}{c|}{\textbf{ALFWorld}}
& \multicolumn{2}{c}{\textbf{WebShop}} \\
\cmidrule(lr){3-9} \cmidrule(lr){10-11}
&
& Pick & Clean & Cool & Look & Heat & Pick2 & All
& Score & Succ. \\
\midrule

\multicolumn{11}{l}{\textbf{\textit{Closed-Source Models}}} \\
Prompting & GPT-4o
& 75.3 & 60.8 & 31.2 & 56.7 & 21.6 & 49.8 & 48.0
& 31.8 & 23.7 \\
Prompting & Gemini-2.5-Pro
& 92.8 & 63.3 & 62.1 & 69.0 & 26.6 & 58.7 & 60.3
& 42.5 & 35.9 \\
\midrule

\multicolumn{11}{l}{\textbf{\textit{Qwen2.5-1.5B-Instruct}}} \\
Prompting & Qwen2.5
& 5.9 & 5.5 & 3.3 & 9.7 & 4.2 & 0.0 & 4.1
& 23.1 & 5.2 \\
Prompting & ReAct
& 17.4 & 20.5 & 15.7 & 6.2 & 7.7 & 2.0 & 12.8
& 40.1 & 11.3 \\
Prompting & Reflexion
& 35.3 & 22.2 & 21.7 & 13.6 & 19.4 & 3.7 & 21.8
& 55.8 & 21.9 \\

RL Training & PPO
& 64.8{\std{3.5}} & 40.5{\std{6.9}} & 57.1{\std{4.9}}
& 60.6{\std{6.6}} & 46.4{\std{4.0}} & 47.4{\std{1.9}}
& 54.4{\std{3.1}}
& 73.8{\std{3.0}} & 51.5{\std{2.9}} \\

RL Training & RLOO
& 88.3{\std{3.0}} & 52.8{\std{8.6}} & 71.0{\std{5.9}}
& 62.8{\std{8.7}} & 66.4{\std{5.5}} & 56.9{\std{4.7}}
& 69.7{\std{2.5}}
& 73.9{\std{5.6}} & 52.1{\std{6.7}} \\

\cmidrule(lr){2-11}

RL Training & GRPO
& 88.2{\std{2.9}} & 72.6{\std{15.9}} & 67.5{\std{7.5}}
& 43.3{\std{18.3}} & 76.6{\std{0.4}} & 51.6{\std{9.5}}
& 71.1{\std{1.6}}
& 81.9{\std{3.5}} & 64.1{\std{2.3}} \\

RL Training & EMPG & 85.5 & 33.5 & 78.9 & 76.2 & 74.7 & 69.1 & 73.7 & 80.4 & 60.8 \\

RL Training & GiGPO
& 97.6{\std{2.4}} & 93.7{\std{1.3}} & \textbf{89.4{\std{0.6}}}
& 65.6{\std{3.1}} & 95.0{\std{5.0}} & 80.6{\std{6.9}}
& 89.5{\std{1.2}}
& 85.2{\std{0.8}} & 74.2{\std{0.8}} \\

RL Training & HCAPO
& 88.6{\std{7.0}} & 97.6{\std{1.8}} & 84.2{\std{0.0}}
& 75.0{\std{0.0}} & 90.7{\std{6.9}} & 74.2{\std{6.9}}
& 87.0{\std{4.1}}
& 83.8{\std{0.7}} & 68.5{\std{1.0}} \\

RL Training & GraphGPO
& 95.6{\std{4.4}} & 97.6{\std{2.4}} & 83.5{\std{3.5}}
& 71.3{\std{8.8}} & \textbf{97.6{\std{2.4}}} & 86.6{\std{2.3}}
& 91.0{\std{2.0}}
& 88.1{\std{0.2}} & \textbf{81.3{\std{0.8}}} \\

\rowcolor{gray!20} \textbf{RL Training} & \textbf{GraphHCA}
& \textbf{100.0{\std{0.0}}} & \textbf{100.0{\std{0.0}}} & 85.7{\std{0.7}}
& \textbf{89.9{\std{2.4}}} & 96.7{\std{3.3}} & \textbf{95.8{\std{4.2}}}
& \textbf{95.7{\std{1.2}}}
& \textbf{89.3{\std{0.1}}} & \textbf{81.3{\std{0.8}}} \\

\midrule
\multicolumn{11}{l}{\textbf{\textit{Qwen2.5-7B-Instruct}}} \\
Prompting & Qwen2.5
& 33.4 & 21.6 & 19.3 & 6.9 & 2.8 & 3.2 & 14.8
& 26.4 & 7.8 \\
Prompting & ReAct
& 48.5 & 35.4 & 34.3 & 13.2 & 18.2 & 17.6 & 31.2
& 46.2 & 19.5 \\
Prompting & Reflexion
& 62.0 & 41.6 & 44.9 & 30.9 & 36.3 & 23.8 & 42.7
& 58.1 & 28.8 \\

RL Training & PPO
& 92.3{\std{4.0}} & 64.0{\std{8.4}} & \textbf{92.5{\std{2.4}}}
& 89.5{\std{7.0}} & 80.3{\std{2.0}} & 68.8{\std{8.3}}
& 80.4{\std{2.7}}
& 81.4{\std{3.1}} & 68.7{\std{5.1}} \\

RL Training & RLOO
& 87.6{\std{4.3}} & 78.2{\std{8.3}} & 87.3{\std{5.8}}
& 81.3{\std{7.6}} & 71.9{\std{5.2}} & 48.9{\std{8.4}}
& 75.5{\std{4.6}}
& 80.3{\std{3.2}} & 65.7{\std{4.0}} \\

\cmidrule(lr){2-11}

RL Training & GRPO
& 88.2{\std{5.9}} & 71.2{\std{9.6}} & 70.0{\std{5.0}} & 50.0{\std{12.5}} & 64.3{\std{2.4}} & 73.7{\std{10.5}} & 73.4{\std{0.8}}
& 82.3{\std{2.7}} & 75.4{\std{2.7}} \\

RL Training & EMPG & 92.9 & 75.2 & 74.8 & 86.3 & 73.7 & 65.3 & 78.5 & 81.0 & 69.3 \\

RL Training & GiGPO
& \textbf{100.0{\std{0.0}}} & 92.3{\std{7.7}} & 85.0{\std{0.0}} & 81.3{\std{6.3}} & \textbf{97.6{\std{2.4}}} & 94.7{\std{0.1}} & 93.8{\std{0.8}}
& 86.9{\std{3.9}} & 82.0{\std{3.1}} \\

RL Training & HCAPO
& 99.1{\std{1.3}} & 97.3{\std{1.9}} & 90.8{\std{6.6}}
& 90.3{\std{2.0}} & 81.8{\std{8.8}} & 81.9{\std{10.0}}
& 91.4{\std{2.3}}
& 85.1{\std{1.3}} & 73.8{\std{2.8}} \\

RL Training & GraphGPO
& 98.5{\std{1.5}} & 98.1{\std{2.7}} & 87.5{\std{3.5}} & 68.8{\std{8.8}} & 95.2{\std{0.0}} & 97.4{\std{3.7}} & 94.1{\std{0.6}}
& 88.6{\std{3.9}} & 82.0{\std{3.1}} \\

\rowcolor{gray!20} \textbf{RL Training }& \textbf{GraphHCA}
& \textbf{100.0{\std{0.0}}} & \textbf{100.0{\std{0.0}}} & 86.7{\std{2.9}} & \textbf{95.8{\std{7.2}}} & 96.8{\std{2.8}} & \textbf{98.3{\std{3.0}}} & \textbf{96.9{\std{0.8}}}
& \textbf{90.5{\std{0.8}}} & \textbf{82.4{\std{1.2}}} \\

\bottomrule
\end{tabular}
}
}
\vspace{-0pt}
\end{table*}


\begin{table}[t]
\centering
\caption{Test performance of VLM agents using Qwen2.5-VL-3B-Instruct on the interactive game environment Sokoban. We report the average success rate (\%) over three random seeds.}
\label{tab:sokoban}
\resizebox{\textwidth}{!}{
\setlength{\tabcolsep}{4.5mm}{
\renewcommand{\arraystretch}{1.4}
\begin{tabular}{cccccc}
\toprule
 & \textbf{Qwen2.5-VL} & \textbf{GRPO} & \textbf{GiGPO} & \textbf{GraphGPO} & \cellcolor{gray!20}\textbf{GraphHCA} \\
\midrule
\textbf{Type} & Prompting & RL Training & RL Training & RL Training & \cellcolor{gray!20}\textbf{RL Training} \\
\textbf{Sokoban [6$\times$6]} & 11.7 & 71.1{\std{3.9}} & 79.0{\std{3.1}} & 81.6{\std{4.03}} & \cellcolor{gray!20}\textbf{84.0{\std{1.2}}} \\
\bottomrule
\end{tabular}
}
}
\vspace{-10pt}
\end{table}

\begin{table}[t]
\centering
\small
\caption{Ablation over the power-mean order $\omega$ of \eqref{eq:mellowmax} with Qwen2.5-1.5B-Instruct on AlfWorld: success rate (\%) averaged over three random seeds. The per-column best is in \textbf{bold}.}
\label{tab:mellowmax}
\setlength{\tabcolsep}{2.0mm}
\renewcommand{\arraystretch}{1.1}
\begin{tabular}{l|ccccccc}
\toprule
$\omega$ & \textbf{Pick} & \textbf{Clean} & \textbf{Cool} & \textbf{Look} & \textbf{Heat} & \textbf{Pick2} & \textbf{All} \\
\midrule
\rowcolor{gray!20} $1$ (GraphHCA)
& \textbf{100.0{\std{0.0}}} & \textbf{100.0{\std{0.0}}} & 85.7{\std{0.7}}
& \textbf{89.9{\std{2.4}}} & 96.7{\std{3.3}} & \textbf{95.8{\std{4.2}}}
& \textbf{95.7{\std{1.2}}} \\
$5$
& 98.5\std{2.1} & \textbf{100.0\std{0.0}} & \textbf{90.0\std{7.1}} & 87.5\std{17.7} & \textbf{100.0\std{0.0}} & 86.8\std{11.2} & 95.3\std{4.5} \\
$10$
& 92.9\std{10.1} & 95.5\std{6.4} & 85.9\std{1.2} & 70.9\std{5.9} & 85.1\std{14.3} & 91.2\std{2.4} & 91.8\std{5.5} \\
\midrule
max ($\omega\!\to\!\infty$)
& 89.7\std{2.1} & 98.1\std{2.7} & 77.5\std{3.5} & 75.0\std{0.0} & 95.2\std{0.0} & 86.9\std{3.7} & 90.1\std{1.1} \\
GraphGPO
& 95.6{\std{4.4}} & 97.6{\std{2.4}} & 83.5{\std{3.5}}
& 71.3{\std{8.8}} & 97.6{\std{2.4}} & 86.6{\std{2.3}}
& 91.0{\std{2.0}} \\
\bottomrule
\end{tabular}
\vspace{-5pt}
\end{table}

\paragraph{Implementation Details.}
Following~\cite{graphgpo}, we use Qwen2.5-1.5B-Instruct and Qwen2.5-7B-Instruct~\citep{qwen2025qwen25technicalreport} as the base LLMs and Qwen2.5-VL-3B-Instruct~\citep{bai2025qwen25vltechnicalreport} as the base VLM.
The agent keeps only the two most recent interaction steps as memory and emits its reasoning in \texttt{<think>} tags before the action in \texttt{<action>} tags~\citep{wei2023chainofthoughtpromptingelicitsreasoning}.
All methods share the same training hyperparameters: group size $N=8$, learning rate $1 \times {10}^{-6}$, and 150 update steps.
GraphHCA adds three quantities of its own: the propagation discount $\bar\gamma=0.95$ of \eqref{eq:phi-backup}, failure floor $\epsilon=1 \times {10}^{-2}$ of \eqref{eq:credit-graph}, and dense weight $w_{\mathrm{s}}=1$ of \eqref{eq:combined-adv}.
Appendix~\ref{app:hyper} reports the more configurations.

\subsection{Main Results}
\label{sec:exp-main}

\paragraph{LLM Agents on ALFWorld and WebShop.}
As shown in Table~\ref{tab:alfworld_webshop}, GraphHCA achieves the best overall ALFWorld success rate and the best result on both WebShop metrics at both model scales.
It improves over the trajectory-level baseline GRPO by 24.6 and 23.5 points of overall ALFWorld success and by 17.2 and 7.0 points of WebShop success rate for the 1.5B and 7B models respectively.
Against the closest baselines GiGPO, HCAPO, and GraphGPO, it still improves overall ALFWorld success by 4.7 to 8.7 points at 1.5B and by 2.8 to 5.5 points at 7B, so the gain comes from the credit rule itself.
The largest gains appear on Look and Pick2, the two longest-horizon subtasks, where it leads the best 1.5B baseline by 13.7 and 9.2 points, while the near-saturated Cool and Heat remain comparable.
Figures~\ref{fig:curve-alfworld} and~\ref{fig:curve-webshop} show the success rate over training steps.
GraphHCA achieves the highest success rate over most of training on both ALFWorld and WebShop.
It also learns fastest, reaching the final success of GRPO in about a third of the training steps.
The success potential provides informative credit when successful rollouts are still rare, which accelerates early learning.

\paragraph{VLM Agent on Sokoban.}
On Sokoban, where observations are rendered images rather than text, GraphHCA attains a success rate of 84.0\%, outperforming GRPO by 12.9 points, GiGPO by 5.0, and GraphGPO by 2.4 (Table~\ref{tab:sokoban}).
It is also the most stable across seeds, with a standard deviation of 1.2 against 3.1 or more for every RL baseline.
Figure~\ref{fig:curve-sokoban} shows where this gain arises: all four methods stay near a $0.2$ success rate for the first twenty updates, GraphHCA is the first to leave that plateau, and the resulting margin holds to the end of training.

\subsection{Ablation Study}
\label{sec:exp-ablation}

\begin{wrapfigure}{r}{0.5\textwidth}
\vspace{-1.5\baselineskip}
\centering
\includegraphics[width=\linewidth]{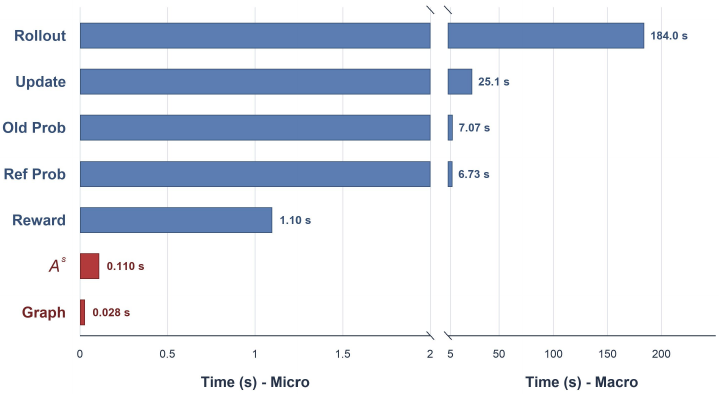}
\caption{Per-iteration runtime breakdown of the training stages. Blue bars denote stages shared by all group-based methods, while red bars denote the overhead of GraphHCA. 
}
\label{fig:runtime}
\vspace{-1.0\baselineskip}
\end{wrapfigure}

\paragraph{Computational overhead.}
GraphHCA introduces no additional learned model and operates solely on the pooled rollout graph constructed from the sampled trajectories.
Graph construction requires a single pass over the observed transitions, and each sweep of the graph backup in~\eqref{eq:phi-backup} likewise scales linearly with the graph size, without additional model inference or gradient computation.
As shown in Figure~\ref{fig:runtime}, rollout and policy update dominate the iteration time at $184.0$\,s and $25.1$\,s, respectively, whereas graph construction and step-level advantage computation require only $0.028$\,s and $0.110$\,s. 
Together, the GraphHCA-specific computation accounts for only $0.062\%$ of the measured per-iteration runtime, providing dense step-level credit with negligible overhead.


\paragraph{Expectation versus best-case aggregation.}
Proposition~\ref{prop:fidelity} relates the difference between GraphHCA and GraphGPO to their aggregation of successor values.
To isolate this factor within the same training pipeline, we replace the expectation in~\eqref{eq:phi-backup} with a $\hat q$-weighted power mean,
\begin{equation}
\hat\Phi_\omega(s)
=
\bar\gamma
\left(
\mathbb E_{a\sim\hat q}
\left[\hat\Phi_\omega(s'_a)^\omega\right]
\right)^{1/\omega},
\qquad
\omega\in[1,\infty),
\label{eq:mellowmax}
\end{equation}
while keeping the boundary conditions, numerical floor, and credit computation unchanged.
At $\omega=1$,~\eqref{eq:mellowmax} recovers the original expectation backup of GraphHCA; increasing $\omega$ progressively emphasizes high-valued successors, and $\omega\to\infty$ yields the max backup whose fixed point is $\bar\gamma^{d(s)}$ on goal-reachable states, which is ranking-equivalent to GraphGPO's shortest-path criterion.
As shown in Table~\ref{tab:mellowmax}, overall success decreases from $95.7\%$ at $\omega=1$ to $95.3\%$ at $\omega=5$, $91.8\%$ at $\omega=10$, and $90.1\%$ under max aggregation.
The max variant closely matches GraphGPO at $91.0\%$.
These results are consistent with Proposition~\ref{prop:fidelity} and indicate that retaining empirical transition frequencies through expectation provides more informative credit than reducing each state to its best observed continuation. Appendix~\ref{app:more_ablations} reports the more ablations.

\section{Conclusion}

We presented \textbf{GraphHCA}, a model-free framework that derives dense step-level credit from the HCA hindsight ratio and estimates it directly from pooled rollout graph. 
By reducing hindsight credit to state-wise success potentials, GraphHCA avoids auxiliary critics, hindsight models, and additional outcome-conditioned policy evaluations while remaining lightweight in practice. 
Extensive experiments on ALFWorld, WebShop, and Sokoban show consistent improvements over GRPO and existing step-level baselines across language and vision-language settings. 
Overall, GraphHCA provides a simple and effective way to extract finer-grained supervision from rollouts already collected for policy optimization.




\bibliography{iclr2027_conference}
\bibliographystyle{iclr2027_conference}

\newpage

\appendix
\section*{Appendix}

\section{Algrithom}
\label{app:algrithom}

\begin{algorithm}[ht]
\caption{Pseudo-code for GraphHCA}
\label{alg:graphhca}
\begin{algorithmic}[1]
\Require Initial policy $\pi_\theta$, task distribution $p(X)$,
propagation discount $\bar\gamma$, numerical floor $\epsilon$,
group size $N$, maximum horizon $T$

\For{each training step}
    \State $\theta_{\mathrm{old}} \gets \theta$

    \Statex \hspace{\algorithmicindent}\textit{// Multi-step rollout}
    \State Sample task $x\sim p(X)$
    \State Roll out $N$ trajectories under $\pi_{\theta_{\mathrm{old}}}$
    until termination or step $T$

    \Statex \hspace{\algorithmicindent}\textit{// Graph estimation}
    \State Construct the pooled rollout graph and empirical weights $\hat q(a\mid s)$
    \State Compute $\hat\Phi$ by iterating~\eqref{eq:phi-backup}
    \State Compute $\hat\Psi(s)=\log\max\{\hat\Phi(s),\epsilon\}$

    \Statex \hspace{\algorithmicindent}\textit{// Hindsight credit assignment}
    \State Compute $\hat\ell_t=\hat\Psi(s_{t+1})-\hat\Psi(s_t)$
    \State Compute same-state advantages $A_t^{\mathrm{s}}$ by~\eqref{eq:step-adv}

    \Statex \hspace{\algorithmicindent}\textit{// Policy optimization}
    \State trajectory-level advantages $A^{\mathrm{T}}$
    \State Compute $\hat A_t$ by~\eqref{eq:combined-adv}
    \State Update $\theta$ using the objective in~\eqref{eq:objective_1}
\EndFor

\end{algorithmic}
\end{algorithm}

Algorithm~\ref{alg:graphhca} summarizes the GraphHCA training procedure.
At each step, a group of rollouts is aggregated into a shared transition graph, from which the discounted graph value and transition-level hindsight credit are computed.
The resulting same-state step advantage is combined with the episode-level advantage for policy optimization.
GraphHCA requires neither a learned critic nor an auxiliary hindsight model or additional outcome-conditioned policy evaluation, as all step-level credit is derived from rollouts already collected for training.
The computational overhead of these graph operations is negligible compared with rollout generation and policy optimization, as shown in Figure~\ref{fig:runtime}.

\section{Proofs}
\label{app:proofs}

This appendix provides proofs for the four propositions in Sec.~\ref{sec:method}: the closed-form hindsight ratio (Sec.~\ref{app:credit}), the contraction and convergence of the graph backup (Sec.~\ref{app:contraction}), the hindsight KL identity and bounded graph credit (Sec.~\ref{app:kl-identity}), and the relation between expectation-based and best-case aggregation (Sec.~\ref{app:fidelity}). We first fix the notation shared across these proofs.

\paragraph{Exact quantities.}
Consider a terminal-goal task with deterministic transitions $s'_a=f(s,a)$ and terminal reward $R(\tau)=\mathbf 1[s_T\in\mathcal G]$. For the exact quantities, we adopt the state-sufficiency assumption in Sec.~\ref{sec:potential}. Trajectories are generated by the behavior policy $q$, and $G=\{s_T\in\mathcal G\}$ denotes the success event. For any state $s$, define
\begin{equation}
\Phi(s)=P(G\mid s),
\qquad
\Phi(s)=1\ \text{for }s\in\mathcal G,\quad
\Phi(s)=0\ \text{for }s\in\mathcal F.
\end{equation}
The corresponding log-scale success potential is
\begin{equation}
\Psi(s)=\log\max\{\Phi(s),\epsilon\},
\qquad \epsilon\in(0,1).
\end{equation}
We write $h(a\mid s,G)=P(a\mid s,G)$ for the success-conditioned hindsight distribution.

\paragraph{Empirical graph quantities.}
For empirical graph quantities, identical environment states are pooled as in Sec.~\ref{sec:graph}, and each observed action is weighted by its empirical frequency
\begin{equation}
\hat q(a\mid s)=\frac{n(s,a)}{n(s)}.
\end{equation}
Let $\mathcal V$ denote the set of graph nodes. We define the backup operator $\mathcal T:\mathbb R^{\mathcal V}\to\mathbb R^{\mathcal V}$ by
\begin{equation}
(\mathcal T\phi)(s)=
\begin{cases}
1, & s\in\mathcal G,\\
0, & s\in\mathcal F,\\
\bar\gamma\displaystyle\sum_a \hat q(a\mid s)\phi(s'_a), & \text{otherwise}.
\end{cases}
\label{eq:T-def}
\end{equation}
A fixed point of $\mathcal T$ therefore satisfies exactly the discounted graph recursion in~\eqref{eq:phi-backup} together with its terminal boundary conditions. We denote this fixed point by $\hat\Phi$ and define the corresponding graph success potential as $\hat\Psi(s)=\log\max\{\hat\Phi(s),\epsilon\}$.

\subsection{Closed-form hindsight ratio (Proposition~\ref{prop:closed-form})}
\label{app:credit}

Deterministic transitions identify the successor state, whose continuation success probability is independent of the preceding history and time step under the state-sufficiency assumption.

\begin{lemma}[Successor identity]
\label{lem:succ}
Under the assumptions of Proposition~\ref{prop:closed-form}, for any non-terminal state $s$ and action $a$ with $q(a\mid s)>0$,
\begin{equation}
P(G\mid s,a)=\Phi(s'_a).
\end{equation}
\end{lemma}

\begin{proof}
Under deterministic transitions, taking action $a$ at state $s$ leads to the unique successor $s'_a=f(s,a)$.
By the state-sufficiency assumption, all histories reaching $s'_a$ have the same conditional probability of eventual success under $q$, irrespective of the time step. Therefore,
\begin{equation}
P(G\mid s,a)=P(G\mid s'_a)=\Phi(s'_a).
\end{equation}
\end{proof}

\begin{proof}[Proof of Proposition~\ref{prop:closed-form}]
For any non-terminal state $s$, marginalizing over the first action gives
\begin{align}
\Phi(s)
&=P(G\mid s) \\
&=\sum_a q(a\mid s)P(G\mid s,a) \\
&=\sum_a q(a\mid s)\Phi(s'_a)
=\mathbb E_{a\sim q(\cdot\mid s)}\big[\Phi(s'_a)\big],
\end{align}
where the third equality follows from Lemma~\ref{lem:succ}. This proves ~\eqref{eq:phi-exact}.

For any state with $\Phi(s)>0$, Bayes' rule gives
\begin{equation}
h(a\mid s,G)=P(a\mid s,G)=\frac{q(a\mid s)P(G\mid s,a)}{P(G\mid s)}.
\end{equation}
For actions in the support of $q(\cdot\mid s)$, dividing by $q(a\mid s)$ and applying Lemma~\ref{lem:succ} yields
\begin{equation}
\rho(s,a)=\frac{P(a\mid s,G)}{q(a\mid s)}
=\frac{P(G\mid s,a)}{P(G\mid s)}
=\frac{\Phi(s'_a)}{\Phi(s)},
\end{equation}
which proves ~\eqref{eq:rho-closed}.
\end{proof}

The corresponding log-scale credit is $\ell(s,a)=\Psi(s'_a)-\Psi(s)$. Whenever the floor in~\eqref{eq:psi-def} is inactive at both states, ~\eqref{eq:rho-closed} gives
\begin{equation}
\ell(s,a)=\log\rho(s,a).
\end{equation}
Independently of whether the floor is active, the potential differences telescope along any realized trajectory:
\begin{equation}
\sum_{t=0}^{T-1}\ell(s_t,a_t)
=
\sum_{t=0}^{T-1}\big[\Psi(s_{t+1})-\Psi(s_t)\big]
=
\Psi(s_T)-\Psi(s_0).
\end{equation}

\subsection{Contraction and convergence (Proposition~\ref{prop:contraction})}
\label{app:contraction}

\begin{proof}[Proof of Proposition~\ref{prop:contraction}]
Let $\phi_1,\phi_2\in\mathbb R^{\mathcal V}$. For terminal states, $\mathcal T\phi_1=\mathcal T\phi_2$ by definition. For any non-terminal state $s$,
\begin{align}
\big|(\mathcal T\phi_1)(s)-(\mathcal T\phi_2)(s)\big|
&=
\bar\gamma\left|\sum_a \hat q(a\mid s)\big[\phi_1(s'_a)-\phi_2(s'_a)\big]\right| \\
&\le
\bar\gamma\sum_a \hat q(a\mid s)\big|\phi_1(s'_a)-\phi_2(s'_a)\big| \\
&\le
\bar\gamma\|\phi_1-\phi_2\|_\infty.
\end{align}
Taking the supremum over $s\in\mathcal V$ gives
\begin{equation}
\|\mathcal T\phi_1-\mathcal T\phi_2\|_\infty
\le
\bar\gamma\|\phi_1-\phi_2\|_\infty.
\end{equation}
Since $\bar\gamma\in(0,1)$, $\mathcal T$ is a contraction on the complete metric space $(\mathbb R^{\mathcal V},\|\cdot\|_\infty)$. Banach's fixed-point theorem therefore guarantees a unique fixed point $\hat\Phi$, and for any initialization $\phi_0$,
\begin{equation}
\|\mathcal T^k\phi_0-\hat\Phi\|_\infty
\le
\bar\gamma^k\|\phi_0-\hat\Phi\|_\infty,
\end{equation}
which establishes geometric convergence, independently of whether the graph contains cycles.

It remains to identify the fixed point. Let $\kappa=\inf\{t\ge0:s_t\in\mathcal G\cup\mathcal F\}$ be the termination time of a graph rollout generated by $a_t\sim\hat q(\cdot\mid s_t)$, and define
\begin{equation}
\tilde\Phi(s)
=
\mathbb E_{\hat q}\!\left[
\bar\gamma^{\,\kappa}
\mathbf 1\{s_\kappa\in\mathcal G\}
\mathbf 1\{\kappa<\infty\}
\,\middle|\,s_0=s
\right].
\end{equation}
The terminal boundary gives $\tilde\Phi=1$ on $\mathcal G$ and $\tilde\Phi=0$ on $\mathcal F$. For any non-terminal $s$, conditioning on the first action yields
\begin{equation}
\tilde\Phi(s)
=
\bar\gamma\sum_a \hat q(a\mid s)\tilde\Phi(s'_a),
\end{equation}
so $\mathcal T\tilde\Phi=\tilde\Phi$. By uniqueness of the fixed point, $\tilde\Phi=\hat\Phi$. Hence, $\hat\Phi$ is exactly the $\bar\gamma$-discounted expected terminal success under the empirical graph policy $\hat q$.
\end{proof}

\subsection{Hindsight KL identity and bounded step credit (Proposition~\ref{prop:kl})}
\label{app:kl-identity}

\begin{proof}[Proof of Proposition~\ref{prop:kl}]
Fix a state satisfying the conditions of Proposition~\ref{prop:kl}.
The floor in~\eqref{eq:psi-def} is inactive at $s$ and at every successor reached with positive probability under $q$.
For every action in the support of $q(\cdot\mid s)$, Proposition~\ref{prop:closed-form} gives
\begin{equation}
\ell(s,a)
=
\log\frac{h(a\mid s,G)}{q(a\mid s)}.
\end{equation}
Taking expectation under $q(\cdot\mid s)$ yields
\begin{align}
\mathbb E_{a\sim q(\cdot\mid s)}[\ell(s,a)]
&=
\sum_{a:q(a\mid s)>0} q(a\mid s)
\log\frac{h(a\mid s,G)}{q(a\mid s)} \\
&=
-\mathrm{KL}\!\left(
q(\cdot\mid s)\,\|\,h(\cdot\mid s,G)
\right)
\le 0,
\end{align}
where equality holds if and only if $q(\cdot\mid s)=h(\cdot\mid s,G)$ almost everywhere.

Independently, for the empirical graph, the terminal boundary and the discounted backup imply $0\le\hat\Phi(s)\le1$ for every graph node. Whether or not the floor is active, it gives
\begin{equation}
\hat\Psi(s)=\log\max\{\hat\Phi(s),\epsilon\}
\in[\log\epsilon,0].
\end{equation}
Therefore, for every observed transition,
\begin{equation}
|\hat\ell(s,a)|
=
|\hat\Psi(s'_a)-\hat\Psi(s)|
\le
\log(1/\epsilon),
\end{equation}
which proves the boundedness of the graph credit.
\end{proof}

\subsection{Expectation versus best-case aggregation (Proposition~\ref{prop:fidelity})}
\label{app:fidelity}

Under the unit-cost convention of~\citet{graphgpo}, define the max-aggregation counterpart of $\mathcal T$ by
\begin{equation}
(\mathcal T^{\max}\phi)(s)=
\begin{cases}
1, & s\in\mathcal G,\\
0, & s\in\mathcal F,\\
\bar\gamma\displaystyle\max_{a:n(s,a)>0}\phi(s'_a), & \text{otherwise}.
\end{cases}
\label{eq:T-max}
\end{equation}
For any $\phi_1,\phi_2\in\mathbb R^{\mathcal V}$, the terminal values cancel, while for any non-terminal state $s$,
\begin{equation}
\big|(\mathcal T^{\max}\phi_1)(s)-(\mathcal T^{\max}\phi_2)(s)\big|
\le
\bar\gamma\max_{a:n(s,a)>0}
\big|\phi_1(s'_a)-\phi_2(s'_a)\big|
\le
\bar\gamma\|\phi_1-\phi_2\|_\infty.
\end{equation}
Hence, $\mathcal T^{\max}$ is a $\bar\gamma$-contraction in the sup norm and therefore admits a unique fixed point, denoted $\Phi^{\max}$.

\begin{proof}[Proof of Proposition~\ref{prop:fidelity}]
For any $\phi$ and non-terminal state $s$,
\begin{equation}
(\mathcal T\phi)(s)
=
\bar\gamma\sum_a\hat q(a\mid s)\phi(s'_a)
\le
\bar\gamma\max_{a:n(s,a)>0}\phi(s'_a)
=
(\mathcal T^{\max}\phi)(s).
\end{equation}
Both operators are monotone. Starting from the same initialization therefore gives $\mathcal T^k\phi_0\le(\mathcal T^{\max})^k\phi_0$ for every $k$, and taking $k\to\infty$ yields
\begin{equation}
\hat\Phi(s)\le\Phi^{\max}(s).
\end{equation}

It remains to identify $\Phi^{\max}$. Let $d(s)$ denote the shortest-path distance from $s$ to $\mathcal G$ on the empirical graph, with $d(s)=0$ for $s\in\mathcal G$. For every goal-reachable non-terminal state,
\begin{equation}
d(s)=1+\min_{a:n(s,a)>0}d(s'_a).
\end{equation}
Define $v(s)=\bar\gamma^{\,d(s)}$ for goal-reachable states and $v(s)=0$ otherwise. The terminal conditions give $v=1$ on $\mathcal G$ and $v=0$ on $\mathcal F$. For every goal-reachable non-terminal state,
\begin{align}
(\mathcal T^{\max}v)(s)
&=
\bar\gamma\max_{a:n(s,a)>0}\bar\gamma^{\,d(s'_a)} \\
&=
\bar\gamma^{\,1+\min_{a:n(s,a)>0}d(s'_a)} \\
&=
\bar\gamma^{\,d(s)}
=
v(s),
\end{align}
where the second equality follows from $\bar\gamma\in(0,1)$. If $s$ is not goal-reachable, none of its observed successors is goal-reachable, so $(\mathcal T^{\max}v)(s)=0=v(s)$. Thus, $v$ is a fixed point of $\mathcal T^{\max}$, and uniqueness implies
\begin{equation}
\Phi^{\max}(s)=\bar\gamma^{\,d(s)}
\qquad
\text{for every goal-reachable }s.
\end{equation}

To characterize equality, define $\Delta(s)=\Phi^{\max}(s)-\hat\Phi(s)\ge0$. For any non-terminal state,
\begin{align}
\Delta(s)
&=
\bar\gamma\left[
\max_a\Phi^{\max}(s'_a)
-
\sum_a\hat q(a\mid s)\hat\Phi(s'_a)
\right] \\
&=
\bar\gamma\left[
\max_a\Phi^{\max}(s'_a)
-
\sum_a\hat q(a\mid s)\Phi^{\max}(s'_a)
\right]
+
\bar\gamma\sum_a\hat q(a\mid s)\Delta(s'_a).
\end{align}
Both terms are non-negative. Hence, $\Delta(s)=0$ if and only if $\hat q(\cdot\mid s)$ is supported on maximizing successors and $\Delta(s'_a)=0$ for every successor reached with positive probability. Recursively applying the same argument gives the stated equality condition at $s$ and along its $\hat q$-reachable continuation.

Finally, for a fixed state $s$, $\hat\Psi(s)$ is independent of the action, so $\hat\ell(s,a)=\hat\Psi(s'_a)-\hat\Psi(s)$ orders the observed actions by their successor potential $\hat\Psi(s'_a)$. Under unit transition cost, GraphGPO uses
\begin{equation}
R^{\mathrm G}(s,a,s'_a)
\propto
\bar\gamma^{\,d(s'_a)+1},
\end{equation}
which favors smaller $d(s'_a)$ because $\bar\gamma\in(0,1)$. Since $\Phi^{\max}(s'_a)=\bar\gamma^{\,d(s'_a)}$, GraphGPO's shortest-path score is ranking-equivalent to the max-aggregation counterpart, whereas GraphHCA evaluates successors through the empirical expectation under $\hat q$. This establishes the expectation-versus-best-case distinction.
\end{proof}

\section{Benchmark Details And Compared Methods}
\label{app:benchmarks}

\paragraph{ALFWorld.}
ALFWorld~\citep{shridhar2020alfworld} aligns the ALFRED household task suite with an interactive text simulator, so the agent must carry out long-horizon, multi-step decision making from textual observations alone.
It contains 3,827 task instances spanning six categories of common household activity, namely Pick, Clean, Cool, Look, Heat, and Pick2, which differ in how many subgoals a task requires and therefore in how long the chain of decisions to be credited is.
At each turn the agent receives a text description of the room together with its inventory, and emits a template action for navigation, object manipulation, or interaction with a receptacle.
Reward is granted only once the full instruction is satisfied, and we report the success rate on each of the six categories as well as the overall success rate.

\paragraph{WebShop.}
WebShop~\citep{yao2022webshop} is a large-scale web-based interactive environment that places the agent in a realistic online shopping scenario over more than 1.1 million real product pages and approximately 12,000 human-written user instructions.
Given an instruction that names a target product with its desired attributes and a price constraint, the agent issues search queries, browses result pages, opens candidate products, selects attribute options, and finally purchases one item.
An episode is scored by how well the purchased item matches the requested attributes, and it counts as a success when that score is maximal, so we report the average task score alongside the average success rate.

\paragraph{Sokoban.}
Sokoban~\citep{SchraderSokoban2018} is an interactive puzzle game in which the agent pushes boxes onto target squares of a grid.
We adopt the $6\times6$ configuration and present each state as a rendered image, so the policy is a vision-language model that grounds its plan in pixels rather than in text.
The game is a stringent test of multi-step planning because pushes are irreversible: a box driven against a wall or into a corner can make the level unsolvable, and the agent receives no signal until the episode ends.
We report the average success rate.

\paragraph{Compared Methods.}
Under an identical agent framework and protocol, we compare GraphHCA against \emph{closed-source LLMs} (GPT-4o~\citep{openai2024gpt4technicalreport}, Gemini-2.5-Pro~\citep{geminiteam2025geminifamilyhighlycapable}), \emph{prompting-based agents} (the base Qwen2.5 checkpoint, ReAct~\citep{yao2023react}, Reflexion~\citep{shinn2024reflexion}), \emph{critic-based RL} (PPO~\citep{schulman2017proximal}), and \emph{critic-free group-based RL} (RLOO~\citep{kool2019buy,ahmadian2024back}, GRPO~\citep{grpo}, EMPG~\citep{empg}, GiGPO~\citep{gigpo}, HCAPO~\citep{hcapo}, GraphGPO~\citep{graphgpo}).
The last family is the closest comparison, as GiGPO, HCAPO, and GraphGPO differ from GraphHCA only in how step-level credit is obtained.
On Sokoban we compare against its group-based members GRPO, GiGPO, and GraphGPO.

\section{Implementation Details and Hyperparameter Sensitivity}
\label{app:hyper}

\subsection{Training configuration}
\label{app:hyper-config}

All experiments are conducted for 150 training steps across all benchmarks and model configurations. For experiments using Qwen2.5-1.5B-Instruct and Qwen2.5-VL-3B-Instruct, training is performed on 4 NVIDIA A100 GPUs with 80\,GB memory each. For experiments using the larger Qwen2.5-7B-Instruct model, we employ 8 NVIDIA A100 GPUs with 80\,GB memory each to accommodate the increased computational and memory requirements.

\subsection{More Ablations}
\label{app:more_ablations}

\begin{table}[t]
\centering
\small
\caption{ALFWorld ablation on the composition of the combined advantage of \eqref{eq:combined-adv} with Qwen2.5-1.5B-Instruct: per-subtask and overall success rate (\%) over three random seeds. The upper block removes one of the two terms, where $w_{\mathrm s}\!=\!0$ recovers GRPO exactly (reproduced from Table~\ref{tab:alfworld_webshop}). The lower block keeps both and varies their relative weight around the default $w_{\mathrm s}\!=\!1$ (highlighted). The per-column best is in \textbf{bold}.}
\label{tab:advantage}
\setlength{\tabcolsep}{1.5mm}
\renewcommand{\arraystretch}{1.1}
\begin{tabular}{l|ccccccc}
\toprule
\textbf{Advantage} & \textbf{Pick} & \textbf{Clean} & \textbf{Cool} & \textbf{Look} & \textbf{Heat} & \textbf{Pick2} & \textbf{All} \\
\midrule
$A^{\mathrm{T}}$ only ($w_{\mathrm s}\!=\!0$, GRPO)
& 88.2{\std{2.9}} & 72.6{\std{15.9}} & 67.5{\std{7.5}}
& 43.3{\std{18.3}} & 76.6{\std{0.4}} & 51.6{\std{9.5}}
& 71.1{\std{1.6}} \\
$A^{\mathrm{step}}$ only
& 97.1{\std{0.05}} & 
\textbf{100.0{\std{0.0}}} & 82.5{\std{2.5}}
& 81.25{\std{6.3}} & \textbf{97.6{\std{2.4}}} & 84.2{\std{5.3}}
& 92.5{\std{0.4}} \\
\midrule
$w_{\mathrm s}=0.5$
& \textbf{100.0{\std{0.0}}} & \textbf{100.0{\std{0.0}}} & 82.5{\std{2.5}}
& 68.8{\std{6.3}} & 95.2{\std{0.0}} & 92.1{\std{2.6}}
& 93.4{\std{0.5}} \\
$w_{\mathrm s}=0.7$
& \textbf{100.0{\std{0.0}}} & \textbf{100.0{\std{0.0}}} & \textbf{95.0{\std{5.0}}}
& 77.1{\std{10.4}} & 92.9{\std{7.2}} & 93.8{\std{0.9}}
& 94.9{\std{0.4}} \\
\rowcolor{gray!20} $w_{\mathrm s}=1$ (default)
& \textbf{100.0{\std{0.0}}} & \textbf{100.0{\std{0.0}}} & 85.7{\std{0.7}}
& \textbf{89.9{\std{2.4}}} & 96.7{\std{3.3}} & \textbf{95.8{\std{4.2}}}
& \textbf{95.7{\std{1.2}}} \\
\bottomrule
\end{tabular}
\end{table}

\paragraph{Composition of the two advantage terms.}
Equation~\eqref{eq:combined-adv} combines the trajectory-level advantage $A^{\mathrm{T}}$ with the state-level advantage $A^{\mathrm{s}}$. Table~\ref{tab:advantage} examines their complementarity by removing each term in turn and varying the weight $w_{\mathrm s}$. 
Setting $w_{\mathrm s}=0$ recovers GRPO and reduces overall success from $95.7\%$ to $71.1\%$, a drop of $24.6$ points. 
Using $A^{\mathrm{s}}$ alone achieves $92.5\%$, but remains $3.2$ points below the full objective, consistent with the role of $A^{\mathrm{T}}$ in providing trajectory-level supervision, including at singleton states where $A^{\mathrm{s}}=0$. 
Among the mixed settings, performance increases from $93.4\%$ at $w_{\mathrm s}=0.5$ to $94.9\%$ at $w_{\mathrm s}=0.7$ and peaks at $95.7\%$ with the default $w_{\mathrm s}=1$. 
These results support the complementary roles of the two advantage terms rather than either term subsuming the other.

\begin{table}[t]
\centering
\small
\caption{ALFWorld ablation over the propagation discount $\bar\gamma$ of \eqref{eq:phi-backup} with Qwen2.5-1.5B-Instruct: per-subtask and overall success rate (\%) over three random seeds. The default $\bar\gamma\!=\!0.95$ is highlighted. Small $\bar\gamma$ drives the potential toward the distance surrogate $\bar\gamma^{\,d(s)}$ of \eqref{eq:fidelity}. The per-column best is in \textbf{bold}.}
\label{tab:gamma}
\setlength{\tabcolsep}{2.4mm}
\renewcommand{\arraystretch}{1.1}
\begin{tabular}{l|ccccccc}
\toprule
$\bar\gamma$ & \textbf{Pick} & \textbf{Clean} & \textbf{Cool} & \textbf{Look} & \textbf{Heat} & \textbf{Pick2} & \textbf{All} \\
\midrule
$0.55$
& 96.4{\std{3.6}} & 95.5{\std{4.6}} & 85.8{\std{0.8}}
& 79.2{\std{4.2}} & \textbf{97.6{\std{2.4}}} & 90.2{\std{4.5}}
& 90.7{\std{3.2}} \\
$0.75$
& 98.6{\std{1.5}} & \textbf{100.0{\std{0.0}}} & \textbf{87.5{\std{2.5}}}
& 87.5{\std{0.0}} & \textbf{97.6{\std{2.4}}} & 92.1{\std{2.7}}
& 93.3{\std{1.6}} \\
\rowcolor{gray!20} $0.95$ (default)
& \textbf{100.0{\std{0.0}}} & \textbf{100.0{\std{0.0}}} & 85.7{\std{0.7}}
& \textbf{89.9{\std{2.4}}} & 96.7{\std{3.3}} & \textbf{95.8{\std{4.2}}}
& \textbf{95.7{\std{1.2}}} \\
\bottomrule
\end{tabular}
\end{table}

\paragraph{Effect of the propagation discount.}
A small propagation discount $\bar\gamma$ in~\eqref{eq:phi-backup} can weaken credit assignment through two related mechanisms.
First, stronger discounting gives greater relative weight to short successful continuations, allowing path length to outweigh differences in success probability.
Even before clipping, the potential can therefore favor a shorter but less reliable route over a longer, more reliable one.
Second, discounting interacts with the floor $\epsilon=10^{-2}$ in $\hat\Psi(s)=\log\max\{\hat\Phi(s),\epsilon\}$.
By~\eqref{eq:fidelity}, $\hat\Phi(s)\le\bar\gamma^{d(s)}$, so states with $d(s)>\log\epsilon/\log\bar\gamma$ are necessarily clipped to $\log\epsilon$.
This occurs from distances of 8 and 17 steps for $\bar\gamma=0.55$ and $0.75$, respectively, well within the rollout horizon $T=50$, erasing potential differences between clipped states.
For $\bar\gamma=0.95$, the discount factor remains above the floor over $T$ steps ($\bar\gamma^{T}\approx0.077>\epsilon$), although low success probabilities can still trigger clipping.
Table~\ref{tab:gamma} is consistent with these mechanisms: overall success rises from $90.7\%$ to $93.3\%$ and $95.7\%$ as $\bar\gamma$ increases from $0.55$ to $0.75$ and $0.95$.
We therefore set $\bar\gamma=0.95$, which reduces distance bias and clipping caused by propagation depth while retaining the contraction guarantee of Proposition~\ref{prop:contraction}.
\begin{table}[t]
\centering
\small
\caption{ALFWorld ablation over the failure floor $\epsilon$ of \eqref{eq:credit-graph} with Qwen2.5-1.5B-Instruct: per-subtask and overall success rate (\%) over three random seeds. The default $\epsilon\!=\!10^{-2}$ is highlighted. The per-column best is in \textbf{bold}.}
\label{tab:floor}
\setlength{\tabcolsep}{2.4mm}
\renewcommand{\arraystretch}{1.1}
\begin{tabular}{l|ccccccc}
\toprule
$\epsilon$ & \textbf{Pick} & \textbf{Clean} & \textbf{Cool} & \textbf{Look} & \textbf{Heat} & \textbf{Pick2} & \textbf{All} \\
\midrule
$10^{-6}$
& \textbf{100.0{\std{0.0}}} & \textbf{100.0{\std{0.0}}} & \textbf{90.0{\std{0.0}}}
& 81.3{\std{6.3}} & 95.3{\std{4.8}} & 89.5{\std{5.3}}
& 94.9{\std{2.0}} \\
$10^{-4}$
& 98.6{\std{1.5}} & \textbf{100.0{\std{0.0}}} & \textbf{90.0{\std{10.0}}}
& 87.5{\std{0.0}} & \textbf{97.6{\std{2.4}}} & 92.1{\std{2.6}}
& 94.1{\std{1.2}} \\
\rowcolor{gray!20} $10^{-2}$ (default)
& \textbf{100.0{\std{0.0}}} & \textbf{100.0{\std{0.0}}} & 85.7{\std{0.7}}
& \textbf{89.9{\std{2.4}}} & 96.7{\std{3.3}} & \textbf{95.8{\std{4.2}}}
& \textbf{95.7{\std{1.2}}} \\
\bottomrule
\end{tabular}
\end{table}

\paragraph{Sensitivity to the numerical floor.}
The floor $\epsilon$ keeps $\hat\Psi(s)=\log\max\{\hat\Phi(s),\epsilon\}$ finite when $\hat\Phi(s)$ approaches zero.
Under the default $\bar\gamma=0.95$ and $T=50$, the attenuation over the full horizon is $\bar\gamma^{T}\approx0.077$, which exceeds every tested floor.
The floor therefore never clips a state because of its propagation depth, and it binds only states whose continuations rarely or never reach the goal.
Changing $\epsilon$ thus shifts the potential of these failure states, while the potential of every state with $\hat\Phi(s)\ge10^{-2}$ remains unchanged.
Consistent with this analysis, Table~\ref{tab:floor} reports overall success of $95.7\%$, $94.1\%$, and $94.9\%$ for $\epsilon=10^{-2}$, $10^{-4}$, and $10^{-6}$, a variation of at most $1.6$ points that is comparable to the standard deviation across seeds.
These results indicate that $\epsilon$ acts as a numerical safeguard rather than a sensitive hyperparameter, provided that it stays below $\bar\gamma^{T}$.

\section{Prompt Templates}
\label{app:templete}

Figures~\ref{fig:alfworld-prompt}, \ref{fig:webshop-prompt}, and~\ref{fig:sokoban-prompt} present the prompt template and an illustrative filled example for ALFWorld, WebShop, and Sokoban, respectively.
All three instantiate a common schema: the agent is given the task instruction, the step count, at most the two most recent observation-action pairs, and the current observation with its admissible actions, and is asked to reason before committing to a single action.
They diverge only where the environment demands it, Sokoban adding a symbol legend and a warning that pushes are irreversible, and offering a fixed set of four moves rather than a state-dependent action set.
Gold text denotes template placeholders, while green and magenta mark the reasoning and action delimiters, respectively.

\begin{figure}[!htbp]
\centering
\begin{prompttpl}{Prompt Template for ALFWorld}
You are an expert agent operating in the ALFRED Embodied Environment.
Your task is to: \promptvar{task\_description}.
Prior to this step, you have already taken \promptvar{step\_count} step(s).
Below are the most recent \promptvar{history\_length} observations and the corresponding actions you took: \promptvar{action\_history}.
You are now at step \promptvar{current\_step} and your current observation is: \promptvar{current\_observation}.
Your admissible actions of the current situation are: [\promptvar{admissible\_actions}].
\par\smallskip
Now it's your turn to take an action. You should first reason step-by-step about the current situation.
This reasoning process MUST be enclosed within \promptthinktags\ tags.
Once you've finished your reasoning, you should choose an admissible action for current step and present it within \promptactiontags\ tags.
\end{prompttpl}

\vspace{8pt}
\begin{prompttpl}[fontupper=\small\rmfamily]{ALFWorld: Filled Example}
You are an expert agent operating in the ALFRED Embodied Environment.
Your task is to: \textbf{put a clean egg in microwave}.
Prior to this step, you have already taken \textbf{3} step(s).
Below are the most recent \textbf{2} observations and the corresponding actions you took:
\par\smallskip
\noindent [t=2] \textit{obs}: The fridge 1 is closed.\\
\noindent [t=2] \textit{act}: open fridge 1\\
\noindent [t=3] \textit{obs}: You open the fridge 1. In it you see a egg 1, a lettuce 1, a mug 1, and a tomato 2.\\
\noindent [t=3] \textit{act}: take egg 1 from fridge 1
\par\smallskip
You are now at step \textbf{4} and your current observation is:
You pick up the egg 1 from the fridge 1.
Your admissible actions of the current situation are:
[go to sinkbasin 1, go to microwave 1, go to countertop 1, put egg 1 in/on fridge 1, examine egg 1].
\par\smallskip
Now it's your turn to take an action. You should first reason step-by-step about the current situation.
This reasoning process MUST be enclosed within \promptthinktags\ tags.
Once you've finished your reasoning, you should choose an admissible action for current step and present it within \promptactiontags\ tags.
\end{prompttpl}
\caption{Prompt template and illustrative filled example for ALFWorld.}
\label{fig:alfworld-prompt}
\end{figure}

\begin{figure}[!htbp]
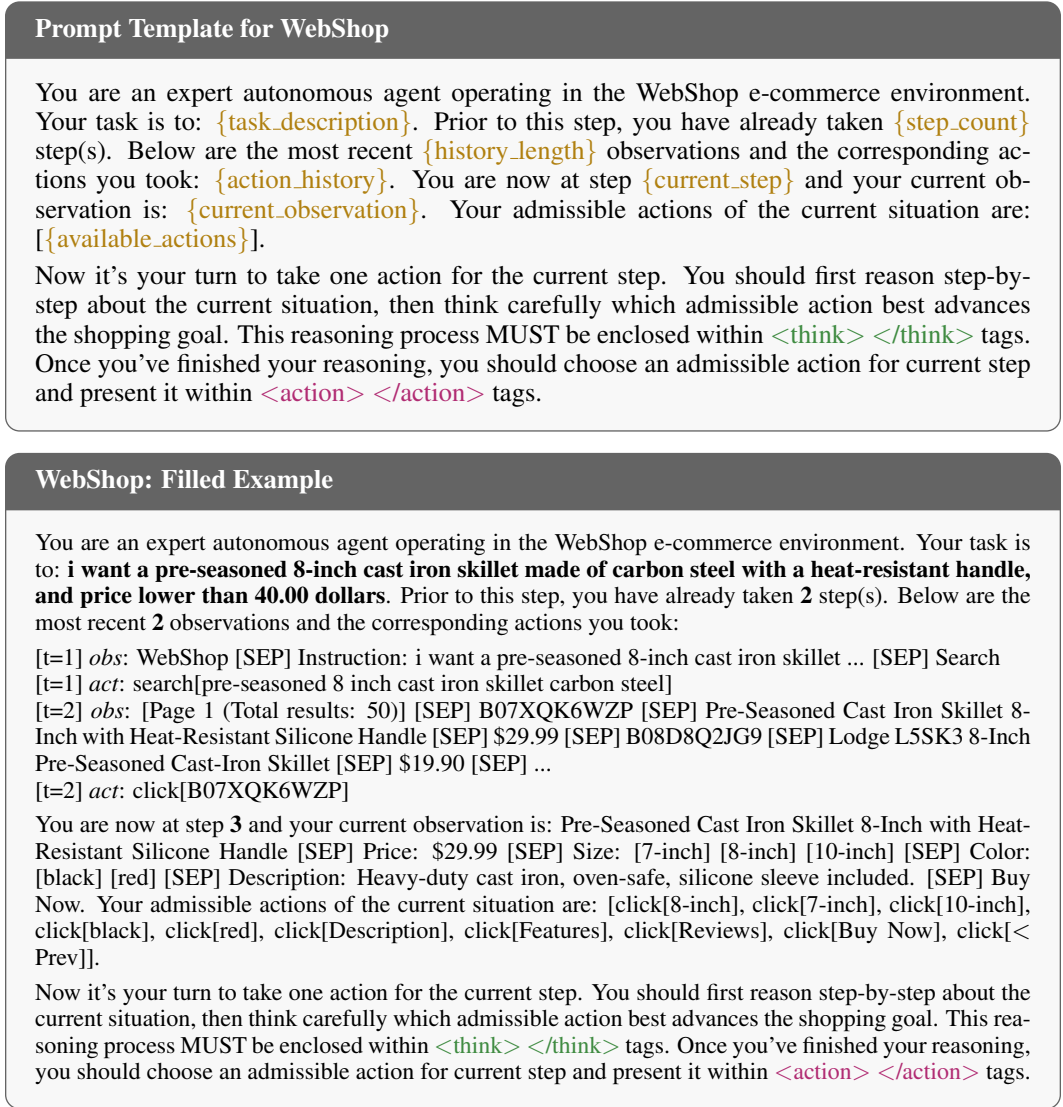

\centering
\begin{prompttpl}{Prompt Template for WebShop}
You are an expert autonomous agent operating in the WebShop e-commerce environment.
Your task is to: \promptvar{task\_description}.
Prior to this step, you have already taken \promptvar{step\_count} step(s).
Below are the most recent \promptvar{history\_length} observations and the corresponding actions you took: \promptvar{action\_history}.
You are now at step \promptvar{current\_step} and your current observation is: \promptvar{current\_observation}.
Your admissible actions of the current situation are: [\promptvar{available\_actions}].
\par\smallskip
Now it's your turn to take one action for the current step.
You should first reason step-by-step about the current situation, then think carefully which admissible action best advances the shopping goal.
This reasoning process MUST be enclosed within \promptthinktags\ tags.
Once you've finished your reasoning, you should choose an admissible action for current step and present it within \promptactiontags\ tags.
\end{prompttpl}

\vspace{8pt}
\begin{prompttpl}[fontupper=\small\rmfamily]{WebShop: Filled Example}
You are an expert autonomous agent operating in the WebShop e-commerce environment.
Your task is to: \textbf{i want a pre-seasoned 8-inch cast iron skillet made of carbon steel with a heat-resistant handle, and price lower than 40.00 dollars}.
Prior to this step, you have already taken \textbf{2} step(s).
Below are the most recent \textbf{2} observations and the corresponding actions you took:
\par\smallskip
\noindent [t=1] \textit{obs}: WebShop [SEP] Instruction: i want a pre-seasoned 8-inch cast iron skillet ... [SEP] Search\\
\noindent [t=1] \textit{act}: search[pre-seasoned 8 inch cast iron skillet carbon steel]\\
\noindent [t=2] \textit{obs}: [Page 1 (Total results: 50)] [SEP] B07XQK6WZP [SEP] Pre-Seasoned Cast Iron Skillet 8-Inch with Heat-Resistant Silicone Handle [SEP] \$29.99 [SEP] B08D8Q2JG9 [SEP] Lodge L5SK3 8-Inch Pre-Seasoned Cast-Iron Skillet [SEP] \$19.90 [SEP] ...\\
\noindent [t=2] \textit{act}: click[B07XQK6WZP]
\par\smallskip
You are now at step \textbf{3} and your current observation is:
Pre-Seasoned Cast Iron Skillet 8-Inch with Heat-Resistant Silicone Handle [SEP] Price: \$29.99 [SEP] Size: [7-inch] [8-inch] [10-inch] [SEP] Color: [black] [red] [SEP] Description: Heavy-duty cast iron, oven-safe, silicone sleeve included. [SEP] Buy Now.
Your admissible actions of the current situation are:
[click[8-inch], click[7-inch], click[10-inch], click[black], click[red], click[Description], click[Features], click[Reviews], click[Buy Now], click[\textless\ Prev]].
\par\smallskip
Now it's your turn to take one action for the current step.
You should first reason step-by-step about the current situation, then think carefully which admissible action best advances the shopping goal.
This reasoning process MUST be enclosed within \promptthinktags\ tags.
Once you've finished your reasoning, you should choose an admissible action for current step and present it within \promptactiontags\ tags.
\end{prompttpl}
\caption{Prompt template and illustrative filled example for WebShop.}
\label{fig:webshop-prompt}
\end{figure}

\begin{figure}[!htbp]
\centering
\begin{prompttpl}{Prompt Template for Sokoban}
You are an expert agent operating in the Sokoban environment.
Your goal is to push every box onto a target square, and the level is solved once no box is left off a target.
\par\smallskip
The grid is drawn with the following symbols:
\begin{itemize}[nosep,leftmargin=1.4em]
\item \texttt{\#} a wall, which neither you nor a box can enter.
\item \texttt{\_} an empty floor square, which you can walk on and push a box onto.
\item \texttt{O} a target square, on which a box has to be placed.
\item \texttt{X} a box.
\item \texttt{P} you, the player.
\item $\surd$ a box that already rests on a target.
\item \texttt{S} you standing on a target.
\end{itemize}
\par\smallskip
Boxes can only be pushed and never pulled, so a box driven into a corner or flattened against a wall may become impossible to move again and can leave the level unsolvable.
Plan ahead and keep every box on a square from which it can still reach a target.
\par\smallskip
Prior to this step, you have already taken \promptvar{step\_count} step(s).
Below are the most recent \promptvar{history\_length} observations and the corresponding actions you took: \promptvar{action\_history}.
You are now at step \promptvar{current\_step} and your current observation is: \promptvar{current\_observation}.
Your admissible actions are: [``up'', ``down'', ``left'', ``right''].
\par\smallskip
Now it's your turn to take one action for the current step.
You should first reason step-by-step about the current situation, locating the boxes and the targets, planning a path that pushes one box toward a target, and checking that the push does not trap that box.
This reasoning process MUST be enclosed within \promptthinktags\ tags.
Once you've finished your reasoning, you should choose an admissible action for current step and present it within \promptactiontags\ tags.
\end{prompttpl}

\vspace{8pt}
\begin{prompttpl}[fontupper=\small\rmfamily]{Sokoban: Filled Example}
You are an expert agent operating in the Sokoban environment.
\textit{(The symbol legend and the rules repeat the template above and are omitted here.)}
\par\smallskip
Prior to this step, you have already taken \textbf{2} step(s).
Below are the most recent \textbf{2} observations and the corresponding actions you took:
\par\smallskip
\noindent
\begin{minipage}[t]{0.48\linewidth}
\centering
[t=1] \textit{obs}\\[2pt]
{\ttfamily\footnotesize\baselineskip=9.5pt
\# \# \# \# \# \#\\
\# \_ \_ O \_ \#\\
\# \_ \_ X \_ \#\\
\# \_ \_ P \_ \#\\
\# \_ X \_ O \#\\
\# \# \# \# \# \#\par}
\vspace{2pt}
[t=1] \textit{act}: up
\end{minipage}%
\hfill
\begin{minipage}[t]{0.48\linewidth}
\centering
[t=2] \textit{obs}\\[2pt]
{\ttfamily\footnotesize\baselineskip=9.5pt
\# \# \# \# \# \#\\
\# \_ \_ $\surd$ \_ \#\\
\# \_ \_ P \_ \#\\
\# \_ \_ \_ \_ \#\\
\# \_ X \_ O \#\\
\# \# \# \# \# \#\par}
\vspace{2pt}
[t=2] \textit{act}: left
\end{minipage}
\par\smallskip
You are now at step \textbf{3} and your current observation is:\par
{\ttfamily\footnotesize\leftskip=1.5em\baselineskip=9.5pt\noindent
\# \# \# \# \# \#\\
\# \_ \_ $\surd$ \_ \#\\
\# \_ P \_ \_ \#\\
\# \_ \_ \_ \_ \#\\
\# \_ X \_ O \#\\
\# \# \# \# \# \#\par}
Your admissible actions are: [``up'', ``down'', ``left'', ``right''].
\par\smallskip
Now it's your turn to take one action for the current step.
You should first reason step-by-step about the current situation, locating the boxes and the targets, planning a path that pushes one box toward a target, and checking that the push does not trap that box.
This reasoning process MUST be enclosed within \promptthinktags\ tags.
Once you've finished your reasoning, you should choose an admissible action for current step and present it within \promptactiontags\ tags.
\end{prompttpl}
\caption{Prompt template and illustrative filled example for Sokoban.}
\label{fig:sokoban-prompt}
\end{figure}

\end{document}